\documentclass{article}
\usepackage[utf8]{inputenc}
\usepackage{amsmath}
\usepackage{amssymb}
\usepackage{mathtools}
\usepackage[numbers, compress]{natbib}
\usepackage{bm}
\usepackage{dsfont}
\usepackage{enumitem}

\usepackage{algorithm}
\usepackage{algpseudocode}
\usepackage{bbm}

\usepackage{color}
\usepackage{hyperref}
\usepackage{url}
\hypersetup{
    colorlinks,
    linkcolor={red!80!black},
    citecolor={blue!90!black},
    urlcolor={blue!90!black}
}

\usepackage[capitalize,noabbrev]{cleveref}
\crefname{equation}{Eq.}{Eqs.}
\crefname{figure}{Fig.}{Figs.}
\crefname{assumption}{Assumption}{Assumptions}

\usepackage[margin=1in]{geometry}
\usepackage{amsthm}
\usepackage{xcolor}
\newtheorem{theorem}{Theorem}

\newtheorem{lemma}[theorem]{Lemma}
\newtheorem{proposition}[theorem]{Proposition}
\newtheorem{definition}{Definition}
\newtheorem{assumption}{Assumption}
\newtheorem{remark}{Remark}

\title{Achieving Fairness in Multi-Agent Markov Decision Processes Using Reinforcement Learning}
\date{May 30, 2023}

\newcommand*\samethanks[1][\value{footnote}]{\footnotemark[#1]}
\author{Peizhong Ju\thanks{Department of ECE, 
 The Ohio State University. Email: \texttt{\{ju.171,ghosh.244\}@osu.edu}} \and Arnob Ghosh\samethanks \and Ness B. Shroff\thanks{Department of ECE and CSE, The Ohio State University. Email: \texttt{shroff.11@osu.edu}}}

\usepackage{dsfont}
\usepackage{bm}
\usepackage{bbm}
\usepackage{algorithm}
\usepackage{algpseudocode}
\newcommand\numberthis{\addtocounter{equation}{1}\tag{\theequation}}

\DeclareMathOperator*{\argmin}{arg\,min}
\newcommand{\defeq}{\coloneqq}
\newcommand{\E}{\mathop{\mathbb{E}}}
\newcommand{\VAR}{\mathsf{VAR}}
\newcommand{\SubOpt}{\mathsf{SubOpt}}

\newcommand{\regret}{\mathsf{Reg}}

\newcommand{\abs}[1]{\left|#1\right|}

\DeclarePairedDelimiter\ceil{\lceil}{\rceil}
\DeclarePairedDelimiter\floor{\lfloor}{\rfloor}

\newcommand{\bigO}{\mathcal{O}}
\newcommand{\bigOTilde}{\tilde{\bigO}}

\newcommand{\nbar}{n_h^{k-1}(s,a)}
\newcommand{\pbar}{\overline{p}_h^{k-1}(s'|s,a)}
\newcommand{\ptilde}{\tilde{p}_h^{k-1}(s'|s,a)}
\newcommand{\pinterval}{\beta_{h,k}^p(s,a,s')}

\newcommand{\rbar}{\overline{r}_{h,(i)}^{k-1}(s,a)}
\newcommand{\rbarOff}{\overline{r}_{h,(i)}(s,a)}
\newcommand{\rtilde}{\tilde{r}_{h,(i)}^{k-1}(s,a)}

\newcommand{\rinterval}{\beta_{h,k}^r(s,a)}

\newcommand{\kthres}{k_{s,a,h}}
\newcommand{\piMix}{\pi^{\text{mix}}}
\begin{document}

\maketitle

\vspace{-0.2in}
\begin{abstract}
Fairness plays a crucial role in various multi-agent systems (e.g., communication networks, financial markets, etc.). Many multi-agent dynamical interactions can be cast as Markov Decision Processes (MDPs). While existing research has focused on studying fairness in \emph{known} environments, the exploration of fairness in such systems for \emph{unknown} environments remains open. In this paper, we propose a  Reinforcement Learning (RL) approach to achieve fairness in multi-agent finite-horizon episodic MDPs. Instead of maximizing the sum of individual agents' value functions, we introduce a fairness function that ensures equitable rewards across agents. Since the classical Bellman's equation does not hold when the sum of individual value functions is not maximized, we cannot use traditional approaches. Instead, in order to explore, we maintain a confidence bound of the unknown environment and then propose an online convex optimization based approach to obtain a policy constrained to this confidence region. We show that such an approach achieves sub-linear regret in terms of the number of episodes. Additionally, we provide a probably approximately correct (PAC) guarantee based on the obtained regret bound. We also propose an offline RL algorithm and bound the optimality gap with respect to the optimal fair solution. To mitigate computational complexity, we introduce a policy-gradient type method for the fair objective. Simulation experiments also demonstrate the efficacy of our approach. 
\end{abstract}
\vspace{-0.1in}
\section{Introduction}
In classical Markov Decision Processes (MDPs), the primary objective is to find a policy that maximizes the reward obtained by a single agent over the course of an episode. However, in numerous real-world applications, decisions made by an agent can have an impact on multiple agents or entities. For instance, in a wireless network scenario, each device aims to maximize its own throughput by increasing its transmission power. However, higher transmission power can lead to interference issues for neighboring terminals. Similarly, consider a situation where two jobs are competing for a single machine; selecting one job results in a penalty or delay for the other job.
The sequential decision-making process as in the above examples  can be cast as a multi-agent episodic MDP where the central decision-maker seeks to obtain the best policy for multiple users or agents over a time horizon. Each user or agent achieves a reward (potentially different) based on the state and action. 

Before delving into the concept of an optimal policy, it is necessary to address what constitutes an optimal policy in the given context. While a particular policy may be good for one agent, it may not be the best choice for another agent. A naive approach could be to maximize the aggregate value functions across all agents, thereby reducing the problem to a classical MDP. However, such an approach may not be considered ``fair'' for all agents involved. To illustrate this, consider a scenario where two jobs are competing for a single machine. If one job offers a higher reward, a central controller that focuses solely on maximizing the aggregate reward may allocate the machine exclusively to the job with the higher reward, causing the job with the lower reward to remain in a waiting state indefinitely. In this paper, our objective is to identify fair decision-making strategies for multi-agent MDP problems, ensuring that all agents are treated equitably.

Drawing inspiration from well-known fairness principles \cite{arrow1965aspects,pratt1978risk,atkinson1970measurement}, we establish a formalization of fairness as a function of the individual value function of agents. Specifically, we concentrate on $\alpha$-fairness, which encompasses both egalitarian or max-min fairness (when $\alpha\rightarrow \infty$) and proportional fairness (when $\alpha=1$). The parameter $\alpha$ allows us to adjust the level of fairness desired. To illustrate this concept, let's consider our example of two jobs with different rewards competing for the same machine. Proportional fairness dictates that the machine should be accessed with equal probability by both the low-reward and high-reward jobs. Conversely, max-min fairness suggests that the job with the higher reward should access the machine with a probability that is inversely proportional to its reward.

In this work, we seek to determine the policy that maximizes the $\alpha$-fairness value of the individual value functions of an MDP. Considering that the knowledge of the environment is usually unknown beforehand in real-world applications, we consider a Reinforcement Learning (RL)-based approach. However, a significant challenge of non-linearity arises since the central controller is not optimizing the sum of the individual value functions, rendering the classical Bellman equation inapplicable. Consequently, conventional techniques such as value-iteration-based or policy-gradient-based approaches cannot be directly employed. To evaluate an online algorithm, regret is a widely-used metric that measures the cumulative performance gap between the online solution in each episode and the optimal solution. Therefore, we aim to develop an algorithm that exhibits sub-linear regret with respect to the $\alpha$-fair solution. Further, since generating new data is costly or impossible for some applications, we also seek to develop a provably-efficient offline fair RL algorithm, i.e., an algorithm that requires no real-time new data. In short, we seek to answer the following--
\begin{center}
    \vspace{-1.5mm}
   \emph{Can we attain a fair RL algorithm with sub-linear regret for multi-agent MDP? Can we develop a provably-efficient fair offline RL algorithm?}
  \vspace{-1.5mm}
\end{center}

\textbf{Our Contributions}: We summarize our contributions in the following:
\begin{itemize}[leftmargin=*]

\item We show that our proposed algorithm achieves $\bigOTilde\left(C_{F}(H^2N^2S\sqrt{AK})\right)$ regret where $H$ is the length of the horizon of each episode, $S$ is the cardinality of state space,  $A$ is the cardinality of the action space, and $K$ is the number of episodes. $C_{F}$ is a parameter determined by the types/parameters of the fairness function. 
 
\item {\em This is the first sub-linear regret bound for the $\alpha$-fairness function in MDP.} We achieve the result by proposing an optimism-based convex optimization framework using state-action occupancy measures. In our algorithm, we use confidence bounds to quantify the error of the estimated reward and transition probability, with which we relax the constraints on possible values of the state-action occupancy measures to encourage exploration. With any convex optimization solver, our proposed algorithm can be solved efficiently in polynomial time. 

\item We also propose a pessimistic version of the optimization problem and establish the theoretical guarantees for the offline fair RL setup. In particular, we construct an MDP with a reward function based on the available data such that the  value function of the constructed MDP is a lower bound of the actual value function for the same policy with high probability. The policy is obtained by solving the convex optimization problem using the occupancy measure on the constructed MDP. We show that the sub-optimality gap of our policy depends on the intrinsic uncertainty multiplied by $C_F$. {\em This is the first result with a theoretical bound for the offline fair RL setup.} 

\item In order to address large state-space problems, we also develop an efficient policy-gradient-based approach that caters to the function approximation setup. 
 
\end{itemize}
\vspace{-0.05in}
\section{Related Work}
\vspace{-0.05in}
\textbf{Multi-objective MDP}: Our work is related to multi-objective RL \cite{roijers2013survey}. Most of the approaches considered a single objective by weighing multiple objectives \cite{van2013scalarized,abels2019dynamic}. Few also proposed algorithms to learn Pareto optimal front \cite{yang2019generalized,mossalam2016multi}. We consider a non-linear function of the multiple value functions and provide a regret bound which was different from the existing approaches. \cite{cheung2019regret} also obtained regret bound for a specific non-linear function of the objectives but cannot capture the $\alpha$-fairness as ours. The algorithm and analysis of \cite{cheung2019regret} are also different compared to ours. Another related approach is Multi-agent RL (MARL) which seeks to learn equilibrium
\cite{li2022learning,jin2021v} in the Markov game. However, our focus is to achieve fairness among the individual value functions. In our setup, the central controller is taking decisions rather than the agents.  Hence, the objective is inherently different, and thus, the algorithms and analysis are also different.

\textbf{Fairness in resource allocation}: Fairness in traditional resource allocation setup has been well studied \cite{mo2000fair,kelly1998rate,lin2006tutorial}. RL-based fair resource allocation decision-making has also been considered for resource allocation \cite{chen2021bringing,hao2023computing,jain2017cooperative,cui2019multi}. However, theoretical guarantees have not been provided. 

\textbf{Fairness in MDP/RL}:
\cite{zhang2014fairness} proposed max-min fairness in MDP, however, the learning component has not been considered.
\cite{joseph2016fairness,liu2017calibrated} considered individual fairness criterion which stipulates that an RL system should make similar decisions for similar individuals or the worst action should not be selected  compared to a better one. \cite{huang2022achieving,schumann2019group,wen2021algorithms} considered a group fairness notion where the main focus is on policy is fair to a group of users (refer to \cite{gajane2022survey} for details). \cite{deng2022reinforcement,metevier2019offline} considered an approach where fairness is modeled as a constraint to be satisfied. \cite{jiang2019learning} proposed an approach where they perturbed the reward to make it fair across the users. In contrast to the above, our setup is different as we seek to achieve fairness in terms of value functions (i.e., long-term return) of different agents.  

\cite{zimmer2021learning,siddique2020learning} considered Gini-fairness across the value functions of the multiple agents, while we focus on different fairness metrics. Besides, the regret bounds have not been provided there. 
\cite{hossain2021fair,bistritz2020my,barman2022fairness,patil2021achieving} considered Nash social welfare (or, proportional fairness) and other fairness notions in multi-armed bandit setup. \cite{docontextual} also considered fairness in the contextual bandit setup. However, we consider an RL setup instead of a bandit setup. The algorithms designed for bandit setup can not be readily extended to the MDP setup.  Further, we consider the generic $\alpha$ fairness concept rather than the proportional-fairness concept. Finally, we provide a fair algorithm in an offline RL setup, which has not been considered in most of the fair RL literature.

The closest to our work is \cite{mandal2022socially} which adopted a welfare-based axiomatic approach and showed regret bound for Nash social welfare, and max-min fairness. In contrast, we considered the $\alpha$-based fairness metric and showed regret bound for the generic value of $\alpha$. Unlike in \cite{mandal2022socially}, our approach admits efficient computation. Further, we provided the PAC guarantee and developed an algorithm for offline fair RL with a theoretical guarantee. Finally, we also developed a policy-gradient-based algorithm that is applicable to large state space as well unlike \cite{mandal2022socially}.

\vspace{-0.05in}
\section{Background: Types of Fairness in Resource Allocation}
\vspace{-0.05in}
 Fairness in resource allocation in multi-agent systems (especially in networks) has been extensively studied \cite{lin2005impact,lin2004joint,eryilmaz2007fair,neely2008fairness}. Specifically, in resource allocation, a feasible solution is any vector $\bm{x}\defeq [x_1\ x_2\ \cdots\ x_N]\in \mathcal{F}\subseteq \mathds{R}^N_{+}$ where $N$ denotes the number of agents, $x_i$ denotes the allocated resource to each agent, and $\mathcal{F}$ denotes a feasible set determined by some constraints. A fair objective is to allocate resources while maintaining some kind of fairness.
As described in \citet{mo2000fair}, the following are some standard definitions of fairness.

\textbf{Proportional Fairness}:
A solution $\bm{x}^*$ is proportional fair when it is feasible and for any other feasible solution $\bm{x}\in \mathcal{F}$, the aggregate of proportional change is non-positive:
\begin{align}\label{eq.prop_fair}
\textstyle
    \sum_{i=1}^N \frac{x_i - x_i^*}{x_i^*}\leq 0.
\end{align}
In particular, for all other allocations, the sum of proportional rate changes with respect to $x^*$ is non-positive. Proportional fairness is widely used in network applications such as scheduling. 

\textbf{Max-min fairness}:
Max-min fairness wants to get a feasible solution that maximizes the minimum resources of all agents, i.e., $\max_{x\in \mathcal{F}}\min_{i} x_i$. For this solution, no agent can get more resources without sacrificing another agent's resources.

\textbf{\texorpdfstring{$(p,\alpha)$}{(p,alpha)}-proportional fairness}:
Let $p=(p_1,\cdots,p_N)$ and $\alpha$ be positive numbers. A solution $\bm{x}^*$ is $(p,\alpha)$-proportionally fair when it is feasible and for any other feasible solution $\bm{x}\in \mathcal{F}$, we have
\begin{align}\label{eq.p_alpha_fair}
\textstyle
    \sum_{i=1}^N p_i \frac{x_i-x_i^*}{{x_i^*}^{\alpha}}\leq 0.
\end{align}
When $p=(1,\cdots,1)$ and $\alpha = 1$, \cref{eq.p_alpha_fair} reduces to \cref{eq.prop_fair}, i.e., the proportionally fair solution. Besides, by Corollary~2 of \cite{mo2000fair}, the solution of $(p,\alpha)$-proportional fair approaches the one of max-min fair as $\alpha \to \infty$. By Lemma~2 of \cite{mo2000fair}, the solution that achieves $(p,\alpha)$-proportional fairness can be solved by
\begin{align*}
    \max_{\bm{x}}&\sum_i p_i f_{\alpha}(x_i),\text{ where } f_{\alpha} (x)=\begin{cases}
    \log x, &\text{if }\alpha =1\\
    (1-\alpha)^{-1} x^{1-\alpha}, & \text{other positive $\alpha$}.
    \end{cases}
\end{align*}
When $p_i=1$ for all $i$, we denote that by $\alpha$-fairness in short, which is widely studied in the networking literature \cite{lan2010axiomatic}. Note that $f_{\alpha}$ is a monotonic increasing function, and concave.

\vspace{-0.05in}
\section{System Model}
\vspace{-0.05in}
\textbf{Multi-agent Finite Horizon MDPs.} Let $\mathcal{M}=(N, \mathcal{S}, \mathcal{A}, r, p, s_1, H)$ be the finite-horizon MDP, where $N$ denotes the number of agents, $\mathcal{A}$ denotes the  action space  with cardinality $A$, $\mathcal{S}$ denotes the  state space with cardinality $S$, and $H$ is a positive integer that denotes the horizon length. At time $h=1,2,\cdots,H$, we let $r_{h,(i)}(s, a)$ denote the non-stationary immediate reward for the $i$-th agent when action is $a\in \mathcal{A}$ at state $s\in \mathcal{S}$. The transition probability is denoted by $p_h(s'|s,a)$. Note that this setup can be easily extended to the scenario where agents are also part of the MDP by letting $\mathcal{A}$ denote the joint action space of all agents.  

\vspace{-0.05in}
\subsection{Value function and Fairness}
\vspace{-0.05in}
The state-action value function of agent $i$ is defined as
\begin{align*}
    Q_{h,(i)}^{\pi}(s,a)\defeq r_{h,(i)}(s,a)+\E\left[\sum_{l=h+1}^{H} r_{l,(i)}(s_l, a_l) | s_h=s, a_h=a, \pi, p\right].
\end{align*}
The $i$-th agent's value function is defined as $V_{h,(i)}^{\pi}(s)\defeq \sum_{a\in \mathcal{A}} \pi_{h}(a|s)Q_{h,(i)}^{\pi}(s,a)$.

To achieve fairness among each agent's return, we optimize a different global value function (instead of $V_{1}^{\pi,\text{sum}}(s)$ that sums up each agent's return):
\begin{align}\label{eq.def_general_V}
\textstyle
    V_1^{\pi,F}(s)\defeq F(V_{1,(1)}^{\pi}(s),V_{1,(2)}^{\pi}(s),\cdots,V_{1,(N)}^{\pi}(s)),
\end{align}
where $F$ is some function of every agent's return that can be chosen for fairness.
Similar to the fairness objective used in resource allocation literature \citep{mo2000fair}, we consider the following three possible options of $F$:
\begin{alignat*}{2}
\textstyle
    F_{\text{max-min}} & \textstyle = \min \{V_{1,(1)}^{\pi}(s),V_{1,(2)}^{\pi}(s),\cdots,V_{1,(N)}^{\pi}(s)\} \quad && \text{(max-min fairness)},\numberthis \label{eq.F_max_min}\\
    \textstyle F_{\text{proportional}} & \textstyle = \sum\nolimits_{i=1}^N \log V_{1,(i)}^{\pi}(s) && \text{(proportional fairness)},\numberthis \label{eq.F_proportional}\\
    \textstyle F_{\alpha}& \textstyle = \sum_{i=1}^N\frac{1}{1-\alpha} {V_{1,(i)}^{\pi}(s)}^{1-\alpha}  && \text{ ($\alpha$ fairness where $\alpha >0$)},\numberthis \label{eq.F_alpha}
\end{alignat*}
Note that we have adopted the $\alpha$-fairness in resource allocation to the $\alpha$-fairness in value functions across the agents. \cite{chen2021bringing} also formalized fairness among value functions in network applications. Recently, \cite{zhang2022proportional} also adopted $\alpha$-fairness to federated learning setup. 
In the rest of this paper, we sometimes remove the superscript $F$ in $V_1^{\pi,F}(s)$ for ease of notation.

{\bf Remarks and Connections:} From (\ref{eq.prop_fair}), maximizing proportional fairness in the value function means that the average relative value function is maximized. In particular, at any other policy, the average relative value function across the agents would be reduced compared to the proportionally fair maximizing policy. \citep{hossain2021fair,mandal2022socially} maximize the product (contrast to the sum) of each agent's value function $\prod_{i=1}^N V_{1,(i)}^{\pi}(s)$ (also, known as Nash social welfare). By taking the logarithm on the product, it is equivalent to $F_{\text{proportional}}$ \cite{kelly1997charging} in \cref{eq.F_proportional} in our case.  
\cite{mandal2022socially} has a regret bound for Nash social welfare. Even though the proportional-fair solution is Nash social welfare solution. The regret bound for the Nash social welfare and for the proportional fair case is not comparable. For example, their bound scales $O(H^{N})$, whereas our regret bound scales as $O(NH^2)$ (shown later in this paper). The proof technique is also different.

When $\alpha=0$, we recover the utilitarian social welfare where the objective is to maximize the sum of the value functions. On the other hand, $\alpha=\infty$ refers to max-min fairness in value functions. By tuning $\alpha$, one can achieve different fairness metrics. 


\vspace{-0.05in}
\subsection{Performance evaluation}
\vspace{-0.05in}
Define the optimal {\em fair} value function corresponding to the optimal policy as
\begin{align*}
\textstyle
    V_1^*(s)=\sup_{\pi}V_1^{\pi}(s).
\end{align*}
The central controller does not know either the probability or the rewards. Rather, it selects a policy $\pi_k$ for $k\in [K]$  episode. 
Without loss of generality, we assume that the initial state $s_1$ for all episodes is the same and fixed. If the initial state $s_1$ is drawn from some distribution, then we can construct an artificial initial state $s_0$ that is fixed for all episodes, and the distribution of the actual initial state $s_1$ determines the transition probability $p_0(s_1|s_0,a)$. We consider the \emph{bandit-feedback setup}, i.e., the central controller can only observe the rewards (of all the agents) corresponding to the encountered state-action pair \cite{agarwal2011stochastic,dani2008stochastic}. We assume the following for the reward:
\begin{assumption}\label[assumption]{as.reward_bound}
The noisy observation of the immediate reward is a random variable $\hat{r}_{h,(i)}(s, a)$, which is in the range $[\frac{\epsilon}{H}, 1]$ almost surely where $\epsilon$ is some positive real number. The mean value of the noisy observation is equal to the true immediate reward, i.e., $\E \hat{r}_{h,(i)}(s, a) = r_{h,(i)}(s, a)$.
\end{assumption}
\begin{remark}\label[remark]{remark.V_bound}
We need $\hat{r}\geq \frac{\epsilon}{H}$ because this guarantee $V_{1,(i)}^{\pi}(s_1)\geq \epsilon >0$, which ensures that \cref{eq.F_proportional,eq.F_alpha} are finite. Also, this makes the functions Lipschitz continuous everywhere.
\end{remark}

 We are interested in minimizing the regret $\regret(K)$ over finite time horizon $K$, given by
\begin{align*}
\textstyle
    \regret(K)\defeq \sum_{k=1}^K \left(V_{1}^*(s_1) - V_{1}^{\pi_k}(s_1)\right).
\end{align*}
The regret characterizes the cumulative sum of the difference at each episode $k=1,2,\cdots, K$ between the fair value function and the optimal fair value function. 


\vspace{-0.05in}
\section{Algorithm}\label{sec:algorithm}
\vspace{-0.05in}
\subsection{Optimal policy with complete information}
\vspace{-0.05in}
Before we characterize the algorithm when the MDP parameters are unknown, we start from the ideal situation where all parameters of the MDP  are known, i.e., complete information. The insight will help us to develop an algorithm for the challenging scenario when the parameters are unknown. 

For the classical objective that maximizes the sum of all agents' returns, the optimal return and policy can be efficiently calculated by backward induction that utilizes the Bellman equation, i.e,
\begin{align}
\textstyle
    V_h^{*,\text{sum}}(s)=\max_{a\in \mathcal{A}}\left\{\sum_{i=1}^N r_{h,(i)}(s,a) + \sum_{s'\in \mathcal{S}}p_h(s'|s,a)V_{h+1}^{*,\text{sum}}(s')\right\},\label{eq.Bellman_sum}
\end{align}
where $V_{H+1}^{*,\text{sum}}(s)=0$ for all $s\in \mathcal{S}$. The reason for \cref{eq.Bellman_sum} is that maximize $\sum_{i=1}^N V_{h, (i)}^{\pi}(s)$ is equivalent to solving another single-agent MDP with immediate reward equal to $\sum_{i=1}^N r_{h,(i)}(s, a)$.

In contrast, such a convenience no longer exists for the fairness objective since \cref{eq.Bellman_sum} relies on the linearity of $V_h^{\pi,\text{sum}}(s)$ w.r.t. $V_{h, (i)}^{\pi}(s)$. 
To solve this problem, we alternatively use an occupancy-measure-based approach which is inspired by \citet{efroni2020exploration}.
Define the occupancy measure
\begin{align}\label{eq.def_q}
    q_h^{\pi}(s,a;p)\defeq \Pr \{s_h=s,a_h=a\ |\ s_1,p,\pi\}.
\end{align}
The occupancy measure defined by \cref{eq.def_q} represents the frequency of the appearance for each state-action pair under the policy $\pi$ on the environment transition probability $p$.
We will omit $p$ in the notation $q_h^{\pi}(s,a;p)$ when the context is clear.

With this definition, each agent's return can be written as a linear function w.r.t. $q_h^{\pi}(s,a)$, i.e.,
\begin{align}\label{eq.def_V_by_occupancy_measure}
\textstyle
    V_{1,(i)}^{\pi}(s_1)=\sum_{s,a,h}r_{h,(i)}(s,a) \cdot  q_h^{\pi}(s,a).
\end{align}
Then we can solve a convex optimization of $q$ (we use $[\cdot]_{i=1,2,\cdots,N}$ to denote $N$ inputs of $F(\cdot)$):
\begin{align}\label{prob.ideal}
    \max_{q\in \mathcal{Q}} \quad &\textstyle F\left(\left[\sum_{s,a,h}r_{h,(i)}(s,a) \cdot  q_h^{\pi}(s,a)\right]_{i=1,2,\cdots, N}\right) \quad \text{ (i.e., $\max_{q\in \mathcal{Q}}V_{1}^{\pi,F}(s_1)$)},
\end{align}
where $\mathcal{Q}$ is a set of linear constraints on $q$ to make sure $q$ is a legit occupancy measure with the transition probability $p$ and initial state $s_1$ (details in \cref{app.proof_regret}).
Since \eqref{prob.ideal} is a convex optimization (proof in \cref{le.convex} in \cref{app.useful_lemmas}), and thus can be solved efficiently in polynomial time. 
After we get the occupancy measure $q$, the corresponding policy can be calculated by $\pi_h(a|s)=\frac{q_h(s,a)}{\sum_{a'}q_h(s,a')}$.

\vspace{-0.05in}
\subsection{Online algorithm with unknown environment}
\vspace{-0.05in}

To construct an online algorithm under the bandit-feedback, a straightforward idea is using the empirical average $\bar{p},\bar{r}$ (precisely defined in \cref{eq.def_empirical,eq.def_empirical_p} in \cref{app.proof_regret}) of the unknown transition probability $p$ and reward $r$ to replace the precise ones in \eqref{prob.ideal}. Due to the imprecision of the empirical average, a common strategy is to introduce some confidence interval to balance exploration and exploitation as done in \cite{efroni2020exploration}.

Define the confidence interval for $\pbar$ as $\pinterval$ such that
\begin{align}\label{eq.p_interval}
    \abs{\pbar - \tilde{p}_{h}(s'|s,a)}\leq \pinterval, \text{ for all }h\in [H-1], s,s'\in \mathcal{S}, a\in \mathcal{A}.
\end{align}
Define the confidence interval for $\rbar$ as $\rinterval$ such that
\begin{align}\label{eq.r_interval}
    \abs{\tilde{r}_{h,(i)}(s,a) - \rbar}\leq \rinterval, \text{ for all } i\in [N], h\in [H], s\in \mathcal{S}, a\in \mathcal{A}.
\end{align}
The expression of $\pinterval$ and $\rinterval$ can be found in \cref{app.proof_regret}.
The overall confidence interval is defined as
\begin{align*}
    M_k\defeq \left\{(\tilde{p},\tilde{r}):\text{Eq.~\eqref{eq.p_interval} and Eq.~\eqref{eq.r_interval}}\right\},
\end{align*}
i.e., the true value of $(p,r)$ is in $M_k$ with high probability.
We now want to use this confidence interval $M_k$ in \eqref{prob.ideal}. A possible way is to replace $p$ and $r$ by $(\tilde{r},\tilde{p})\in M_k$ and view $\tilde{r},\tilde{p}$ as decision variables. Thus, the objective in \eqref{prob.ideal} now becomes 
\begin{align}
    \max_{(\tilde{r},\tilde{p})\in M_k,\ q\in \mathcal{Q}} V_{1}^{\pi, F}(s_1).\label{prob.interval}
\end{align}
It is indeed an optimistic solution compared to the real optimal solution because we relax the value of $r$ and $p$ in such optimization (which leads to better objective value). However, it is no longer a convex optimization problem because now $r$ and $p$ are decision variables. To turn such an optimization into a convex one, we need to determine the value of $\tilde{r}$ and $\tilde{p}$ beforehand.
To that end, notice the monotonicity of the objective with respect to $\tilde{r}$ (proof in \cref{le.convex} in \cref{app.useful_lemmas}). Thus, without affecting the solution of \eqref{prob.interval}, we can let
\begin{align}\label{eq.optimistic_reward}
    \tilde{r}_{h,(i)}(s,a)=\rbar+\rinterval.
\end{align}
Now, we only need to determine the value of $\tilde{p}$. To that end,
consider the state-action-next-state occupancy measure $z_h^{\pi}(s,a,s';p)\defeq p_h(s'|s,a) q_h^{\pi}(s,a;p)$. Considering Eq.~\eqref{eq.p_interval}, we only need
\begin{equation}\label{eq.z_constratint}
\begin{alignedat}{2}
    &\textstyle z_h(s,a,s')\leq \left(\pbar+\pinterval\right)\sum_{y\in \mathcal{S}} z_h(s,a,y) && \text{ for all }h\in [H-1],s,a,s',\\
    &\textstyle z_h(s,a,s')\geq \left(\pbar-\pinterval\right)\sum_{y\in \mathcal{S}} z_h(s,a,y) && \text{ for all }h\in [H-1],s,a,s'.
    \vspace{-0.05in}
\end{alignedat} 
\end{equation}

Now we are ready to solve \eqref{prob.interval} by convex optimization. Specifically,
at the $h$-th step during the $k$-th iteration, we solve the following extended convex optimization:
\begin{align}\label{op.extended}
    \max_{z\in \mathcal{Z}}\textstyle F\left( \left[ \sum_{s,a,h,s'} \left(\rbar+\rinterval\right)\cdot  z_h(s,a,s')\right]_{i=1,2,\cdots, N}\right)\text{ s.t. \cref{eq.z_constratint}},
\end{align}
where $\mathcal{Z}$ is a set of constraints that ensures $z$ is a legit state-action-next-state occupancy measure given initial state $s_1$ (details in \cref{app.proof_regret}).
Once we have solved $z$, we can recover the policy by $\pi_{k,h}(a|s)=\frac{\sum_{s'}z_h(s,a,s')}{\sum_{a',s'}z_h(s,a',s')}$. The whole algorithm is summarized in \cref{alg.main}. \cite{mandal2022socially}  developed an algorithm using the state-action occupancy measure. However, the algorithm in \cite{mandal2022socially} relies on an optimization problem with infinite variables, which does not always have a polynomial solver. In contrast, our approach requires only finite variables and is more efficient. 
\setlength{\textfloatsep}{4pt}
\begin{algorithm}
\caption{Online FairMARL}\label{alg.main}
\begin{algorithmic}[1]
\For{$k=1,2,\cdots,K$}
\State Calculate the empirical average $\pbar$, and $\rbar$.
\State Calculate the confidence intervals $\pinterval$ and $\rinterval$.
\State Compute policy $\pi_k$ by solving \eqref{op.extended}.
\State Execute the policy $\pi_k$.
\EndFor
\end{algorithmic}
\end{algorithm}

\begin{theorem}\label{th.main}
With probability $1-\delta$, we have
\begin{align*}
    \regret(K)=C_{F}\cdot\left(\bigOTilde(H^2N S \sqrt{AK})+\bigOTilde(HN^2S^{3/2}A)+\bigOTilde(H^2NS^2A)\right),
\end{align*}
where $C_{F}$ is a constant determined by the type of fairness. Specifically,
\begin{align*}
    C_{F}=\begin{cases}
        \epsilon^{-\alpha} &\text{when $F=F_{\alpha}$ ($\alpha$ fairness) }\\
        \epsilon^{-1} &\text{when $F=F_{\text{proportional}}$ (proportional fairness)}\\
        1/N &\text{when $F=F_{\text{max-min}}$ (max-min fairness)}
    \end{cases}.
\end{align*}
The notation $\bigOTilde(\cdot)$ ignores logarithm terms (such as $\log \frac{1}{\delta}$).
\vspace{-0.05in}
\end{theorem}
Proof of \cref{th.main} is in \cref{app.proof_regret}.
\begin{remark}
For max-min fairness, the requirement $\hat{r}\geq \frac{\epsilon}{H}$ in \cref{as.reward_bound} can be relaxed. \vspace{-0.05in}
\end{remark}

To the best of our knowledge, this is the first sub-linear regret for $\alpha$-fair RL. When $\alpha=0$, we recover the single-agent regret (scaled by $N$) as it is equivalent to the MDP where the reward is $\sum_{i}r_{h,(i)}$. The constant $C_F$ decreases as $\alpha$ increases. For the max-min fairness, our result matches that of \cite{mandal2022socially}, although we use a different algorithm compared with \cite{mandal2022socially}.




\textbf{From regret to PAC guarantee}: The probably approximately correct (PAC) guarantee shows how many samples are needed to find an $\varepsilon$-optimal policy $\pi$ satisfying $V_1^*(s_1)-V_1^{\pi}(s_1)\leq \varepsilon$ \citep{jin2018q,valiant1984theory}.
Similar to Section~3.1 in \citet{jin2018q}, in order to get the probably approximately correct (PAC) guarantee from regret, we can randomly select $\pi=\pi_k$ for $k=1,2,\cdots, K$. We define such a policy as $\piMix$. However, since $V_1^{\pi}(s)$ is not linear w.r.t. the immediate reward $r$, generally $V_1^{\piMix,F}(s)\neq \frac{1}{K}\sum_{k=1}^K V_1^{\pi_k,F}(s)$. Therefore, compared with \citet{jin2018q}, some additional derivation is needed to achieve PAC guarantee from regret in our case. In particular, from Jensen's inequality (since $F$ is concave), $V_1^{\piMix}(s_1)\geq \frac{1}{K}\sum_{k=1}^K V_1^{\pi_k}(s_1)$. Using the above, we obtain 

\begin{theorem}\label{th.PAC}
To find $\varepsilon$-optimal policy, with high probability, it suffices to have $C$ number of samples where
\vspace{-0.1in}
\begin{align*}
    C=C_F \max\left\{\bigOTilde(H^5 N^2 S^2 A/\varepsilon^2),\ \bigOTilde(H^3 N^4 S^3 A^2 /\varepsilon^2),\ \bigOTilde(H^3 N^2 S^4 A^2 / \varepsilon^2)\right\}.
\end{align*}
\vspace{-0.1in}
\end{theorem}
Proof of \cref{th.PAC} is in \cref{proof.PAC}.

\vspace{-0.05in}
\subsection{Proof outline of Theorem~\ref{th.main}} \vspace{-0.05in}

First, because of the optimism of $q$ (i.e., $\pi_k$), $\tilde{r}$, $\tilde{p}$ in \eqref{prob.interval}, we have $V_1^{\pi_k}(s_1;\tilde{r},\tilde{p})\geq V_1^*(s_1;r,p)$ when $r,p\in M_k$ (which happens with high probability). By the montonocity property of $F(\cdot)$, we thus have $\sum_{k=1}^K \left(V_1^{*,F}(s_1)-V_1^{\pi_k,F}(s_1)\right) \leq \sum_{k=1}^K \left(V_1^{\pi_k,F}(s_1;\tilde{r},\tilde{p}) -V_1^{\pi_k,F}(s_1;r,p) \right)$.

Second, to handle the non-linearity of the fair objective $F$, we bound $V_1^{\pi_k,F}(s_1;\tilde{r},\tilde{p})-V_1^{\pi_k,F}(s_1;r,p)$ by $C_F\sum_{i=1}^N\abs{V_{1,(i)}^{\pi_k}(s_1;\tilde{r},\tilde{p})-V_{1,(i)}^{\pi_k}(s_1;r,p)}$ or by $C_F \cdot N \max_{i\in [N]}\abs{V_{1,(i)}^{\pi_k}(s_1;\tilde{r},\tilde{p})-V_{1,(i)}^{\pi_k}(s_1;r,p)}$. The value of $C_F$ is determined by the Lipchitz constant of the fairness objective function $F$ or the property of the max-min operator.

Third, $\abs{V_{1,(i)}^{\pi_k}(s_1;\tilde{r},\tilde{p})-V_{1,(i)}^{\pi_k}(s_1;r,p)}$  is the gap of individual's return caused by the difference between $(\tilde{r},\tilde{p})$ and $(r,p)$. We can bound the gap using tools like Azuma-Hoeffding inequality. 

\vspace{-0.05in}
\section{Offline Fair MARL}\label{sec:offline}
As mentioned in the Introduction, we develop an offline algorithm because for some applications generating new data may not be feasible.  In an offline setting, the learner is given a dataset and it needs to compute a policy only based on this given dataset. One can not employ a policy and measure its return. Due to this difference, instead of optimism, pessimism is optimal for standard MDP \cite{xie2021bellman,jin2021pessimism}. 
We develop an offline fair algorithm and analyze its sub-optimality gap. 
Before delving into the result, we need to have some assumptions about the data collection process.
\begin{assumption}
The dataset $\mathcal{D}=\{r_{h,(i)}^{\tau},x_{h}^{\tau},a_{h}^{\tau}\}_{h\in [H],\tau\in [K],i\in [N]}$ is compliant with the underlying MDP, i.e., $\forall i$,
\begin{align}
& \mathbb{P}_{\mathcal{D}}(r^{\tau}_{h,(i)}=r^{\prime}_i,x_{h+1}^{\tau}=x^{\prime}|\{(x_{h}^j,a_h^j)\}_{j=1}^{\tau},\{(r_{h,(i)}^j,x_{h+1}^j)\}_{j=1}^{\tau-1})=\nonumber\\
& \mathbb{P}(r_{h,(i)}(s_h,a_h)=r^{\prime}_i,s_{h+1}=x^{\prime}|s_h=x^{\tau}_h,a_h=a^{\tau}_h).\nonumber
\end{align}
\end{assumption}
The above assumption is satisfied when the data is collected by interacting with the environment and the policy is only updated at the end of an episode. \cite{jin2021pessimism} also uses a similar assumption. Similar to the online algorithm, we denote the empirical estimation $\overline{p}$ and $\overline{r}$ on $p$ and $r$, respectively, for the dataset $\mathcal{D}$. We first define the uncertainty quantizer for the data set which we use to construct MDP with pessimistic reward.
\begin{definition}\label[definition]{def.offline_event}
We define the set $\mathcal{E}$ as the $\delta$-uncenrtainty quantifier with respect to the dataset $\mathcal{D}$ as--
\begin{align*}
 \mathcal{E}=  \{b^r_h(s,a,\delta),b_h^{p}(s,a,s^{\prime},\delta):\abs{\rbarOff-r_{h,(i)}(s,a)}\leq b^r_h(s,a,\delta)\nonumber\\
 |\overline{p}_h(s^{\prime}|s,a)-p_h(s^{\prime}|s,a)|\leq b_h^{p}(s,a,s^{\prime},\delta), \forall i, s,a,h\}
\end{align*} 
such that $\mathbb{P}_{\mathcal{D}}(\mathcal{E})\geq 1-\delta$.
\end{definition}
The values of $b^r_h(s,a,\delta),b_h^{p}(s,a,s^{\prime},\delta)$ are given in Appendix~\ref{proof:offline}. They are related to $\beta_{h,k}^r$, and $\beta_{h,k}^p$. The only difference is that the empirical estimation now depends on the dataset rather than the obtained information till episode $k$ in the online version.

We can show that with probability $1-\delta$, for any $V_{h,(i)}$
\begin{align}
\vspace{-0.05in}
|\mathbb{P}_h[V_{h,(i)}]-\bar{\mathbb{P}}_h[V_{h,(i)}](s,a)|=\sum_{s^{\prime}}|(\overline{p}_h(s^{\prime}|s,a)-p_h(s^{\prime}|s,a))V_{h,(i)}(s^{\prime})|\leq H\sum_{s^{\prime}}b_h^p(s,a,s^{\prime},\delta).\nonumber
\vspace{-0.05in}
\end{align}
We then define the pessimistic reward $\underline{r}_{h,(i)}$ as
\begin{align*}
\vspace{-0.05in}
    \underline{r}_{h,(i)}(s,a)\defeq \rbarOff - b^r_h(s,a,\delta)-H\sum_{s^{\prime}}b^p_h(s,a,s^{\prime},\delta).
    \vspace{-0.05in}
\end{align*}
Note that we have also subtracted $H\sum_{s^{\prime}}b_h^p(s,a,s^{\prime})$ in order to ensure the value function attained for the MDP with reward $\underline{r}_{h,(i)}$ and empirical probability $\overline{p}_h$ is less than the value function corresponding to the original MDP parameters for the same policy with probability $1-\delta$, i.e., ensure pessimism.

To bound the suboptimality, we need an additional assumption that each agent's return under pessimistic reward $\underline{r}$ should be positive and shouldn't be too small. Specifically, we need the following assumption.
\begin{assumption}\label[assumption]{as.offline_reward}
$V_{1,(i)}(s_1,\underline{r},\overline{p})\geq \epsilon$ for all $i$.
\end{assumption}
The above assumption is required to apply the Lipschitz continuous property (see \cref{le.F_C_F} in \cref{app.useful_lemmas}). If $\underline{r}_{h,(i)}(s,a)\geq \epsilon/H$ for every $h,i,s,a$, then the above Assumption is trivially satisfied. Also, our analysis would go through using a slightly larger $\underline{r}$ since the true reward value is greater than or equal to $\epsilon/H$. In particular, we can set $\underline{r}_{h,(i)}(s,a)=\max\{\overline{r}_{h,(i)}(s,a)-b_h^r(s,a,\delta),\epsilon/H\}-H\sum_{s^{\prime}}b_h^p(s,a,s^{\prime})$. Hence, it is clear that  \cref{as.offline_reward} is more likely to hold when the uncertainty on the estimation of  $p$ in the offline data is small. This is reasonable because when the uncertainty is high, it is unlikely to bound the regret, especially since some fair objectives $F$ are unbounded when any agent's return is near $0$.


Our proposed offline algorithm is solving the following convex optimization:
\begin{alignat*}{2}
    \max_{q\in \overline{\mathcal{Q}}} \quad & F\left(\left[ \sum_{h,s,a}\underline{r}_{h,(i)}(s,a)q_h(s,a)\right]_{i=1,2,\cdots,N}\right),\numberthis \label{op.offline}
\end{alignat*}
where $\overline{\mathcal{Q}}$ is the same as $\mathcal{Q}$ in \eqref{prob.ideal} but with $\overline{p}$ instead of $p$.
Similar to the online algorithm, we still use the occupancy measure $q$ to construct a convex optimization problem. However, compared with the online version, a key difference is that we use a pessimistic reward  instead of the optimistic reward  in the objective. Further, the MDP is based on the empirical probability $\overline{p}$, unlike the online setup where we allow the probability to take value within the confidence interval. 

\vspace{-0.05in}
\subsection{Performance guarantee of the offline algorithm}
\vspace{-0.05in}
We denote the solution of \cref{op.offline} as $\hat{q}$ and the corresponding policy as $\hat{\pi}$. The suboptimality of any policy $\pi$ is defined by
\vspace{-0.1in}
\begin{align*}\SubOpt(\pi;s)\defeq &V^{\pi^*}_1(s;r,p) - V^{\pi}_1(s;r,p).
\end{align*}
\vspace{-0.05in}
\begin{theorem}\label{th.offline}
Given offline data $\mathcal{D}$, with probability $1-\delta$\begin{align}
\SubOpt(\hat{\pi};s_1)\leq 2NC_F\mathbb{E}_{\pi^*}[\underbrace{\sum_h((b_h^r(s_h,a_h,\delta)+H\sum_{s'}b_h^p(s_h,a_h,s',\delta)))}_{\mathrm{Intrinsic  -Uncertainty}}]
\end{align}
\vspace{-0.1in}
\end{theorem}
To the best of our knowledge, this is the first offline RL result for the $\alpha$-fairness function. In the standard single-agent MDP, the result also depends on the $\delta$ uncertainty quantifier term and intrinsic uncertainty term that constitutes information theoretic lower limit on optimality-gap \cite{jin2021pessimism}. Here, it is scaled by $N$ and $C_F$. $C_{F}$ is the Lipschitz constant which depends on $\alpha$-fairness function. Further, if the dataset $\mathcal{D}$ has good coverage over the optimal policy, then the $\mathrm{Intrinsic-Uncertainty}$ term is small.
\vspace{-0.05in}
\subsection{Proof sketch of Theorem~\ref{th.offline}}
\vspace{-0.05in}
We have
\begin{align}
    \vspace{-0.05in}
    \SubOpt(\hat{\pi};s)=&\underbrace{\left(V^{\pi^*,F}_1(s;r,p)-V_1^{\pi^*,F}(s;\underline{r},\overline{p})\right)}_{\text{Term 1}}+\underbrace{\left(V^{\pi^*,F}_1(s;\underline{r},\overline{p})-V_1^{\hat{\pi},F}(s;\underline{r},\overline{p})\right)}_{\text{Term 2}}\nonumber\\
& +\underbrace{\left(V_1^{\hat{\pi},F}(s;\underline{r},\overline{p})-V_1^{\hat{\pi},F}(s;r,p)\right)}_{\text{Term 3}}\label{eq.offline_decomp}
\vspace{-0.05in}
\end{align}
Term~2 of \cref{eq.offline_decomp} is non-positive because $\hat{\pi}$ is the solution of \cref{op.offline}. In standard offline RL literature \cite{jin2021pessimism,xie2021bellman}, Term~3 of \cref{eq.offline_decomp} is non-positive because of the pessimism which is proved using Bellman's property. However, since Bellman's property does not hold, we cannot use the standard technique. Rather, we use the Lipschitz property of $F(\cdot)$ to show that
\begin{align}\label{eq:offline}
\textstyle
\left(V^{\pi^*,F}_1(s;r,p)-V_1^{\pi^*,F}(s;\underline{r},\overline{p})\right)\leq \sum_{i=1}^NC_F|V^{\pi^*}_{1,(i)}(s;r,p)-V^{\pi^*}_{1,(i)}(s;\underline{r},\overline{p})|.
\end{align}
The right-hand side then can be bounded by the Value-difference lemma. 

We can bound Term~3 as 
$\left(V_1^{\hat{\pi},F}(s;\tilde{r},\overline{p})-V_1^{\hat{\pi},F}(s;r,p)\right)\leq 0
$
because of the pessimistic reward $\underline{r}$ and the fact that $F(\cdot)$ is monotone increasing w.r.t. $r$. 

\vspace{-0.05in}
\section{Fair Online Policy Gradient}\label{sec:policy_gradient}
\vspace{-0.05in}
In Section~\ref{sec:algorithm}, we developed a convex-optimization-based algorithm to obtain sub-linear regret. However, the decision variable and the constraints scale with the cardinality of the state space. In order to develop an algorithm for large state space, generally function approximation-based approaches are used to approximate the $Q$ function or value function. In this section, we develop a policy-gradient-based approach that caters to such a function approximation-based approach. 

Considering a trajectory $\tau=(s_h^{\tau},a_h^{\tau},\hat{\bm{r}}_h^{\tau})_{h=1,2,\cdots,H}$ where $\hat{\bm{r}}_h^{\tau}=(\hat{r}_{h,(1)}^{\tau},\cdots,\hat{r}_{h,(N)}^{\tau})$ denotes the noisy observation of immediate reward for all agents, we define the return for the $i$-th agent as
   $ R_{(i)}(\tau)\defeq \sum_{h=1}^H \hat{r}_{h,(i)}^{\tau}$.
To calculate the gradient of the fair objective, we can apply the chain rule of $\nabla_{\bm{\theta}}F(\cdot)$. We use proportional fairness $F_{\text{proportional}}$ as an example of calculating the gradient:
\begin{align*}
\nabla_{\theta}V_{1}^{\pi_{\bm{\theta}},F}(s_1)=\nabla_{\bm{\theta}}\sum_{i=1}^N \log V_{1, (i)}^{\pi_{\bm{\theta}}}(s_1)=\sum_{i=1}^N \frac{\partial \log(V_{1, (i)}^{\pi_{\bm{\theta}}}(s_1))}{\partial V_{1, (i)}^{\pi_{\bm{\theta}}}(s_1)}\nabla_{\bm{\theta}} V_{1, (i)}^{\pi_{\bm{\theta}}}(s_1)=\sum_{i=1}^N \frac{\nabla_{\bm{\theta}}V_{1, (i)}^{\pi_{\bm{\theta}}}(s_1)}{V_{1, (i)}^{\pi_{\bm{\theta}}}(s_1)}.
\end{align*}
It is known that $\nabla_{\bm{\theta}}V_{1, (i)}^{\pi_{\bm{\theta}}}(s_1)= \E_{\tau} [R_{(i)}(\tau) \log \pi_{\bm{\theta}}(a_h^{\tau} | s_h^{\tau})]$ and $V_{1, (i)}^{\pi_{\bm{\theta}}}(s_1)=\E_{\tau}[R_{(i)}(\tau)]$. By using the empirical average to replace $\E_{\tau}$, we can get an unbiased estimator of gradient w.r.t. $\bm{\theta}$ as follows: 
\begin{align*}
\bm{g}_{\text{proportional}} = \sum_{i=1}^N \frac{\sum_{\tau\in \mathcal{D}}\sum_{h=1}^H R_{(i)}(\tau)\nabla_{\bm{\theta}}\log \pi_{\bm{\theta}}(a_h^{\tau} | s_h^{\tau})}{\sum_{\tau\in \mathcal{D}}R_{(i)}(\tau)}.
\end{align*}
For other types of fairness, we can use a similar method. The final expression of the gradient, the rest part of the algorithm, and other related details are in \cref{app.gradient}. Note that we can extend this approach to the natural policy-gradient, actor-critic method, and baseline-based approach. Characterization of the convergence rate of the approaches are beyond the scope of this paper. Interested readers can refer to \cite{zhang2020global,agarwal2021theory,mei2020global} for convergence analysis of standard policy gradient.

\vspace{-0.05in}
\section{Numerical Results}
\vspace{-0.1in}
We have conducted experiments on randomly generated MDP environments. We observe that Algorithm 1 indeed achieves sub-linear regret
Please see \cref{app:simulation} for details.
\vspace{-0.05in}
\section{Conclusion and Future Work}
\vspace{-0.05in}
In this paper, we develop convex-optimization-based algorithms for both the online and offline fair RL with provable  performance guarantee. Potential future directions include studying decentralized fair MARL algorithms and other policy gradient methods along with their convergence. Developing provably-efficient fair RL algorithms beyond tabular setup constitutes a future research direction. 

\bibliographystyle{unsrtnat}
\bibliography{ref}

\newpage

\appendix

\begin{center}
\textbf{\large Supplemental Material}
\end{center}

\section{Useful Lemmas}\label{app.useful_lemmas}

\begin{lemma}\label[lemma]{le.max_min_eq}
Let $x_1,x_2,\cdots,x_N$ and $y_1,y_2,\cdots,y_N$ be real numbers. We must have
\begin{align*}
    \abs{\min_{i\in \{1,2,\cdots,N\}} x_i - \min_{j\in \{1,2,\cdots,N\}} y_j} \leq \max_{i\in [N]}\abs{x_i-y_i}.
\end{align*}
\end{lemma}
\begin{proof}
Without loss of generality, we let $\min_i x_i \geq \min_j y_j$. For any $i^*\in \argmin_i x_i$ and any $j^* \in \argmin_j y_j$, we have
\begin{align*}
    x_{i^*} \leq x_{j^*},
\end{align*}
which implies that
\begin{align}
    x_{i^*}-y_{j^*}\leq x_{j^*} - y_{j^*}\leq \abs{x_{j^*} - y_{j^*}}.\label{eq.temp_051801}
\end{align}
We thus have
\begin{align*}
    \abs{\min_i x_i -\min_j y_j}=& \min_i x_i - \min_j y_j\\
    =&x_{i^*}-y_{j^*}\\
    \leq & \abs{x_{j^*}-y_{j^*}}\quad \text{ (by \cref{eq.temp_051801})}\\
    \leq & \max_{i}\abs{x_i - b_i}.
\end{align*}
The result of this lemma thus follows.
\end{proof}

\begin{lemma}\label[lemma]{le.F_C_F}
Recall the definition of $C_F$ in \cref{th.main}. When $x_i,y_i\geq \epsilon > 0$ for all $i\in [N]$, we must have
\begin{align*}
    \abs{F(x_1,x_2,\cdots,x_N)-F(y_1,y_2,\cdots,y_N)} \leq N \cdot C_F \max_{i\in [N]}\abs{x_i-y_i}.
\end{align*}
\end{lemma}
\begin{proof}
When $F=F_{\text{proportional}}$, we have
\begin{align*}
    &\abs{F(x_1,x_2,\cdots,x_N)-F(y_1,y_2,\cdots,y_N)}\\
    =&\abs{\sum_{i=1}^N \log x_i - \log y_i }\\
    \leq & \sum_{i=1}^N \abs{\log x_i - \log y_i}\ \text{ (by the triangle inequality)}\\
    \leq & N \max_{i\in [N]} \abs{\log x_i - \log y_i}\\
    \leq & N \frac{1}{\epsilon} \max_{i\in[N]}\abs{x_i-y_i}.
\end{align*}
The last step is by the Lipschitz continuity of $\log(\cdot)$ in the domain $[\epsilon,\infty)$, where $\frac{1}{\epsilon}$ is the corresponding Lipschitz constant.

Similarly, when $F=F_{\alpha}$, since the Lipschitz constant of $\frac{(\cdot)^{1-\alpha}}{1-\alpha}$ in the domain $[\epsilon,\infty)$ is $\epsilon^{-\alpha}$, we can show that
\begin{align*}
    \abs{F(x_1,x_2,\cdots,x_N)-F(y_1,y_2,\cdots,y_N)} \leq N \cdot \epsilon^{-\alpha} \max_{i\in [N]} \abs{x_i-y_i}.
\end{align*}

When $F=F_{\text{max-min}}$, by \cref{le.max_min_eq}, we have
\begin{align*}
    \abs{F(x_1,x_2,\cdots,x_N)-F(y_1,y_2,\cdots,y_N)} \leq & \max_{i\in [N]} \abs{x_i-y_i}\\
    =& N \frac{1}{N}\cdot \max_{i\in [N]} \abs{x_i-y_i}.
\end{align*}
Notice that in this case $x_i,y_i\geq \epsilon$ is not needed.

In summary, the result of this lemma thus follows.
\end{proof}

\begin{lemma}\label[lemma]{le.convex}
\eqref{prob.ideal} is a convex optimization whose value is monotone increasing w.r.t. the immediate reward $r$.
\end{lemma}
\begin{proof}
Notice that the constraints of \eqref{prob.ideal} are linear, we only need to prove that the fair objectives in \cref{eq.F_max_min,eq.F_proportional,eq.F_alpha} are concave  w.r.t. state-action occupancy measure $q$ and state-action-state occupancy measure $z$.
Notice that $V_{1,(i)}^{\pi}$ is a weighted sum of $q$ and $z$, in order to prove the concavity, it remains to show that \cref{eq.F_max_min,eq.F_proportional,eq.F_alpha} are concave w.r.t. $V_{1,(i)}^{\pi}$. Notice that $\min(\cdot)$ and $\log(\cdot)$ are concave. We only need to verify the concavity of $F_{\alpha}$. Since
\begin{align*}
    \frac{\partial^2 \frac{1}{1-\alpha}x^{1-\alpha}}{\partial x^2}=-\alpha x ^{-\alpha - 1},
\end{align*}
which is non-positive when $x\geq 0$. Thus, we have also proven the concavity of $F_{\alpha}$. Therefore, we have proven that \eqref{prob.ideal} is a convex optimization.

Notice that $r$ only appears in the objective (i.e., the constraints do not have $r$), and all $F_{\text{max-min}}$, $F_{\text{proportional}}$, $F_{\alpha}$ are monotone increasing w.r.t. $r$. Thus, the value of \eqref{prob.ideal} is monotone increasing w.r.t $r$.

The result of this lemma thus follows.
\end{proof}

\begin{lemma}[Hoeffding's inequality]\label[lemma]{le.Hoeffding}
Let $Z_1,Z_2,\cdots,Z_n$ be \emph{i.i.d.} samples of a random variable $Z\in [0,1]$. For any $\tilde{\delta} >0$, we must have
\begin{align*}
    \Pr\left\{\abs{\E Z - \frac{1}{n}\sum_{i=1}^n Z_i} \leq \sqrt{\frac{\ln (2/ \tilde{\delta})}{2n}}\right\}\geq 1 - \tilde{\delta}.
\end{align*}
\end{lemma}

\begin{lemma}[empirical Bernstein inequality (Theorem~4 of \cite{maurer2009empirical})]\label{le.Bernstein}
Let $Z_1,Z_2,\cdots,Z_n$ be \emph{i.i.d.} samples of a random variable $Z\in [0,1]$. For any $\tilde{\delta} >0$, we must have
\begin{align*}
    \Pr\left\{\abs{\E Z - \frac{1}{n}\sum_{i=1}^n Z_i}\leq \sqrt{\frac{2V_n \ln (4/\tilde{\delta})}{n}}+\frac{7\ln (4/\tilde{\delta})}{3(n-1)}\right\}\geq 1 - \tilde{\delta},
\end{align*}
where $\VAR_n$ is the sample variance
\begin{align}\label{eq.def_sample_var}
    \VAR_n=\frac{1}{n(n-1)}\sum_{1\leq i\leq j\leq n}(Z_i-Z_j)^2.
\end{align}
\end{lemma}

\begin{lemma}[Lemma~F.4 of \cite{dann2017unifying}]\label{le.unify}
Let $\mathcal{F}_i$ for $i=1,2,\cdots$ be a filtration and $X_1,X_2,\cdots$ be a sequence of Bernoulli random variables with $\Pr\{X_i=1|\mathcal{F}_{i-1}\}=P_i$ with $P_i$ being $F_{i-1}$-measurable and $X_i$ being $F_i$ measurable. For any $W\geq 0$, It holds that
\begin{align*}
    \Pr\left\{\text{exist }n: \sum_{t=1}^n X_t < \sum_{t=1}^n \frac{P_t}{2}-W\right\}\leq e^{-W}.
\end{align*}
\end{lemma}

The following is the standard value difference lemma. Its proof can be found in, e.g., \cite{dann2017unifying}, Lemma~E.15.
\begin{lemma}[Value difference lemma]\label[lemma]{le.value_difference}
For any two MDPs $M'$ and $M''$ with rewards $r'$ and $r''$ and transition probabilities $P'$ and $P''$, the difference in value functions with respect to the same policy $\pi$ can be written as
\begin{align*}
    V_i'(s)-V_i''(s)=&\E_{P'',\pi}\left[\sum_{t=i}^H\left(r'(s_t,a_t,t)-r''(s_t,a_t,t)\right)\bigg| s_i=s\right]\\
    &+\E_{P'',\pi}\left[\sum_{t=i}^H\sum_{\tilde{s}}\left(P'_t(\tilde{s}|s_t,a_t)-P''_t(\tilde{s}|s_t,a_t))^T V'_{t+1}(\tilde{s})\right)\right].
\end{align*}
\end{lemma}


\section{Details in Section~\ref{sec:algorithm}}\label{app.proof_regret}


\subsection{About optimization problems}

In this subsection, we will show details of the optimization problems \eqref{prob.ideal}, \eqref{prob.interval}, and \eqref{op.extended}.

Let $\mu(s)$ denote the probability of the initial state $s$. (for a fixed initial state $s_1$, then $\mu(s)$ equals to $1$ for $s=s_1$ while equals to $0$ otherwise.)

\noindent\textbf{About $\mathcal{Q}$ (constraints of $q$):}

The following are the constraints that make $q$ a legit state-action occupancy measure, i.e., the definition of $\mathcal{Q}$:
\begin{equation}\label{prob.ideal_full}
\begin{alignedat}{2}
    & \sum_a q_h(s,a)=\sum_{s',a'}p_{h-1}(s|s',a')q_{h-1}(s',a')\quad &&\text{for all } s\in \mathcal{S},h\in [H]\setminus\{1\}\\
    & q_h(s,a)\geq 0\quad &&\text{for all } s\in \mathcal{S},a\in\mathcal{A},h\in[H]\\
    & \sum_a q_1(s,a)=\mu(s)\quad &&\text{for all } s\in \mathcal{S}.
\end{alignedat}
\end{equation}
The constraint $\sum_{s,a} q_h(s,a)=1$ is redundant because the first and the third constraint imply $\sum_{s,a} q_h(s,a)=1$ for all $h\in [H]$.

\noindent\textbf{About $\mathcal{Z}$ and \eqref{eq.z_constratint} (constraints of $z$):}

Recall the definition of $z$
\begin{align}\label{eq.temp_051902}
    z_h(s,a,s')\defeq p_h(s'|s,a) q_h(s,a).
\end{align}
By summing over the next state $s'$ on both sides of \cref{eq.temp_051902}, we have 
\begin{align}\label{eq.temp_051901}
    \sum_{s'} z_h(s,a,s')=\sum_{s'} p_h(s'|s,a) q_h(s,a) = q_h(s,a).
\end{align}

Summing over $a$ on both sides of \cref{eq.temp_051901}, we have
\begin{align}\label{eq.temp_051903}
    \sum_{a,s'} z_h(s,a,s')=\sum_a q_h(s,a).
\end{align}

Thus, we can rewrite the constraints $\mathcal{Q}$ \eqref{prob.ideal_full} in the form of $z$ as $\mathcal{Z}$:
\begin{equation}\label{eq.first_constraint_in_z}
\begin{alignedat}{2}
    &\sum_{a,s'} z_h(s,a,s')=\sum_{s',a'}z_{h-1}(s',a',s)\quad &&\text{for all } s\in \mathcal{S},h\in [H]\setminus\{1\} ,\\
    & z_h(s,a,s')\geq 0\quad &&\text{for all }s,a,s',h,\\
    &\sum_{a,s'} z_1(s,a,s')=\mu(s)\quad &&\text{for all } s\in \mathcal{S}.
\end{alignedat} 
\end{equation}
In \cref{eq.first_constraint_in_z}, we get the first constraint by plugging \cref{eq.temp_051903} into the left side of the first constraint of \cref{prob.ideal_full} while plugging \cref{eq.temp_051902} into the right side. We get the third constraint by plugging \cref{eq.temp_051903} into the third constraint of \cref{prob.ideal_full}.

Notice that by replacing $q$ by $z$, we have one additional requirement \cref{eq.temp_051902}. 
Using \cref{eq.temp_051901} to replace $q_h(s,a)$ in \cref{eq.temp_051902}, we can express \cref{eq.temp_051902} as
\begin{align}\label{eq.temp_051904}
    z_h(s,a,s')=p_h(s'|s,a)\cdot \sum_{y\in \mathcal{S}} z_h(s,a,y).
\end{align}
By \cref{eq.p_interval,eq.temp_051904}, we have the constraint \cref{eq.z_constratint} (used in the optimization problem \eqref{op.extended}).



We now show that \eqref{prob.interval} is equivalent to \eqref{op.extended} by the following proposition.

\begin{proposition}
\eqref{prob.interval} and \eqref{op.extended} are equivalent. In other words, the optimal value of the objective of \eqref{prob.interval} is equal to the optimal value of the objective of \eqref{op.extended}.
\end{proposition}
\begin{proof}
By the monotonicity w.r.t. $r$ shown in \cref{le.convex}, we know that the optimal choice of $\tilde{r}$ in \eqref{prob.interval} is $\overline{r}+\beta^r$. It remains to show that the effect of choosing optimal of $\tilde{p}$ in \eqref{prob.interval} is equivalent to \cref{eq.z_constratint}. To that end, notice that $\tilde{p}$ does not appear in the objective. 
Thus, we only need to focus on how $\tilde{p}$ affects the constraints of $z$.
Notice that among all constraints in \cref{eq.first_constraint_in_z,eq.temp_051904}, the only one that connects $p$ and $z$ is \cref{eq.temp_051904}. Since the optimal $\tilde{p}$ in \eqref{prob.interval} must be in the confidence interval $[\overline{p}-\beta^p,\ \overline{p}+\beta^p]$, we know that the optimal objective value by using \cref{eq.z_constratint} is at least as good as the one by using the optimal $\tilde{p}$ in \cref{eq.temp_051904}. From another aspect, For the optimal $z$ get by \cref{eq.z_constratint}, we can always construct $\tilde{p}$ which is in the confidence interval $[\overline{p}-\beta^p,\ \overline{p}+\beta^p]$ by letting
\begin{align*}
    \tilde{p}_h(s'|s,a)=
    \begin{cases}
        \frac{z_h(s,a,s')}{\sum_{y\in \mathcal{S}}z_h(s,a,y)} \quad &\text{ if }\sum_{y\in \mathcal{S}}z_h(s,a,y)\neq 0,\\
        \overline{p}\quad & \text{ if }\sum_{y\in \mathcal{S}}z_h(s,a,y)=0.
    \end{cases}
\end{align*}
which implies that the optimal objective value by using optimal $\tilde{p}$ in \cref{eq.temp_051904} is at least as good as the one by using \cref{eq.z_constratint}. The equivalent of these two different approaches is thus follows.
\end{proof}


\subsection{Proof of Theorem~\ref{th.main}}

To prove \cref{th.main}, we will first introduce the good event and its probability, then prove a regret bound under the good event. Some auxiliary lemmas are needed in the proof. We list them at the end of this subsection.

\textbf{Failure events and the good event}

Define the empirical average of the transition probability and the immediate reward at the $k$-th iteration of \cref{alg.main} as
\begin{align*}
    & \nbar\defeq \sum_{k'=1}^{k-1}\mathbbm{1}\left(s_h^{k'}=s,a_h^{k'}=a\right),\numberthis\label{eq.def_n}\\
    & \pbar\defeq \frac{\sum_{k'=1}^{k-1}\mathbbm{1}\left(s_h^{k'}=s,a_h^{k'}=a,s_{h+1}^{k'}=s'\right)}{\max\{\nbar,\ 1\}},\numberthis\label{eq.def_empirical_p}\\
    & \rbar\defeq \frac{\sum_{k'=1}^{k-1} \hat{r}_{h,(i)}^{k'}(s,a)\cdot\mathbbm{1}\left(s_h^{k'}=s,a_h^{k'}=a\right)}{\max\{\nbar,\ 1\}}. \numberthis \label{eq.def_empirical}
\end{align*}

Define
\begin{align*}
    &\pinterval\defeq \sqrt{\frac{4\pbar(1-\pbar)L_{\delta}^p}{\max\{\nbar,1\}}}+\frac{14 L_{\delta}^p}{3\max\{\nbar,1\}},\numberthis \label{eq.def_pinterval}\\
    &\rinterval\defeq \sqrt{\frac{L_{\delta}^r}{\max\{\nbar,1\}}}.\numberthis \label{eq.def_rinterval}
\end{align*}
where $L_{\delta}^p\defeq \ln \frac{12S^2AHK}{\delta}$ and $L_{\delta}^r\defeq 2\ln \frac{3SAHNK}{\delta}$.

We define the following failure events based on confidence intervals in Eq.~\eqref{eq.p_interval} and Eq.~\eqref{eq.r_interval}.

\begin{align*}
    &G^p\defeq \left\{\text{exist some }s,a,s',h,k\text{ such that }\abs{\pbar - p_{h}(s'|s,a)}\geq \pinterval\right\},\\
    &G^n\defeq \left\{\text{exist some }s,a,h,k\text{ such that }\nbar\leq \frac{1}{2}\sum_{j<k}q_h^{\pi_j}(s,a)-\ln \frac{3SAH}{\delta}\right\},\numberthis\label{eq.def_Gn}\\
    &G^r\defeq\left\{\text{exist some }s,a,h,i,k\text{ such that }\abs{\rbar-r_{h,(i)}^{k-1}(s,a)}\geq \rinterval\right\}.
\end{align*}
Intuitively, $G^p$ denotes the case where the transition probability is out of the confidence interval, $G^n$ denotes the case where the empirical occupancy measure deviates from the actual occupancy measure, and $G^r$ denotes the case where the empirical reward is out of the confidence interval. The following lemma estimates the probability of those failure events.

\begin{lemma}\label{lem:gp}
We have
\begin{align*}
    \Pr\{G^p\}\leq \frac{\delta}{3}.
\end{align*}
\end{lemma}
\begin{proof}
We first focus on the situation on fixed $s,a,s',h,k$.
If $\nbar\in\{0,1\}$, then
\begin{align*}
    \pinterval=\sqrt{4\pbar(1-\pbar)L_{\delta}^p}+\frac{14L_{\delta}^p}{3}\geq \frac{14L_{\delta}^p}{3}\geq \frac{14\ln 4}{3}>2.
\end{align*}
Thus, we have
\begin{align}\label{eq.temp_101701}
    \Pr\left\{\abs{\pbar - p_{h}(s'|s,a)}\geq\pinterval\ |\ \nbar\in \{0,1\}\right\}=0\leq \frac{\delta}{3S^2AHK}.
\end{align}
Now we consider the case of $\nbar\geq 2$.
We define $b_1,b_2,\cdots,b_{\nbar}$, where each term is $\mathbbm{1}\left(s_{h+1}^{k'}=s'\right)$ under the condition $s_h^{k'}=s$ and $a_h^{k'}=a$ for $k=1,2,\cdots,k-1$ (recall the definition of $\nbar$ in Eq.~\eqref{eq.def_n}). Thus, $b_1,b_2,\cdots,b_{\nbar}$ are $\nbar$ \emph{i.i.d.} samples of Bernoulli distribution with the parameter of the (success) probability $p_h(s'|s,a)$. Therefore, the sample variance (defined in Eq.~\eqref{eq.def_sample_var}) of these $\nbar$ samples is equal to
\begin{align*}
    &\frac{1}{\nbar(\nbar-1)}\sum_{1\leq i\leq j \leq \nbar-1} (b_i-b_j)^2\\
    =&\frac{\sum_{k'=1}^{k-1}\mathbbm{1}\left(s_h^{k'}=s,a_h^{k'}=a,s_{h+1}^{k'}=s'\right)\cdot \sum_{k'=1}^{k-1}\mathbbm{1}\left(s_h^{k'}=s,a_h^{k'}=a,s_{h+1}^{k'}\neq s'\right)}{\nbar(\nbar-1)}\\
    =&\frac{\nbar}{\nbar-1}\pbar(1-\pbar)\text{ (by Eq.~\eqref{eq.def_empirical_p})}\\
    \leq & 2\pbar(1-\pbar).
\end{align*}
Thus, by Lemma~\ref{le.Bernstein} (where $\tilde{\delta}=\frac{\delta}{3S^2AHK}$), for fixed $s,a,s',h,k$, we have
\begin{align*}
    \Pr&\left\{\abs{\pbar - p_{h}(s'|s,a)}\geq \sqrt{\frac{4\pbar(1-\pbar)\ln \frac{12S^2AHK}{\delta}}{\nbar}}\right.\\
    &\quad \left.+\frac{7\ln \frac{12S^2AHK}{\delta}}{3(\nbar-1)}\right\}\leq \frac{\delta}{3S^2AHK}.
\end{align*}
Notice that $\frac{7}{3(\nbar-1)}\leq \frac{14}{3}$ when $\nbar\geq 2$. We thus have
\begin{align}\label{eq.temp_101702}
    \Pr\left\{\abs{\pbar - p_{h}(s'|s,a)}\geq\pinterval\ |\ \nbar\geq 2\right\}\leq \frac{\delta}{3S^2AHK}.
\end{align}
Combining Eq.~\eqref{eq.temp_101701} and Eq.~\eqref{eq.temp_101702}, we thus have 
\begin{align*}
    \Pr\left\{\abs{\pbar - p_{h}(s'|s,a)}\geq\pinterval\right\}\leq \frac{\delta}{3S^2AHK}.
\end{align*}
Applying the union bound by traversing all $s,a,s',h,k$, we thus have
\begin{align*}
    \Pr\{G^p\}\leq \frac{\delta}{3}.
\end{align*}
The result of this lemma thus follows.
\end{proof}

\begin{lemma}\label{lem:gn}
We have
\begin{align*}
    \Pr\{G^n\}\leq \frac{\delta}{3}.
\end{align*}
\end{lemma}
\begin{proof}
For fixed $s,a,h$, by Lemma~\ref{le.unify} (letting $W=\ln\frac{3SAH}{\delta}$), we have
\begin{align*}
    \Pr\left\{\text{exist }k\text{ such that }\nbar\leq \frac{1}{2}\sum_{j<k}q_h^{\pi_j}(s,a|p)-\ln \frac{3SAH}{\delta}\right\}\leq \frac{\delta}{3SAH}.
\end{align*}
Applying the union bound by traversing all $s,a,h$, the result of this lemma thus follows.
\end{proof}

\begin{lemma}\label{lem:gr}
We have
\begin{align*}
    \Pr\{G^r\}\leq \frac{\delta}{3}.
\end{align*}
\end{lemma}
\begin{proof}
For fixed $s,a,h,i,k$, by Lemma~\ref{le.Hoeffding}, we have
\begin{align*}
    \Pr\left\{\abs{\rbar-r_{h,(i)}^{k-1}(s,a)}\geq \sqrt{\frac{L_{\delta}^r}{\nbar}}\right\}\leq \frac{\delta}{3SAHNK}
\end{align*}
Applying the union bound by traversing all $s,a,h,i,k$, the result of this lemma thus follows.
\end{proof}

\noindent\textbf{The regret bound under the good event}

\begin{lemma}\label[lemma]{le.reg}
If outside the union of all failure events $G^p\cup G^n \cup G^r$, then we must have
\begin{align*}
    \regret(K)\leq & 4C_F \sqrt{L_{\delta}^r \ln (4+K)}HN\sqrt{SAK} + 2 C_F \sqrt{L_{\delta}^r}\left(4\ln \frac{SAH}{\delta'}+5\right)HNSA  \\
    &+8 C_F \sqrt{L_{\delta}^p\ln (4+K)} H^2 N S \sqrt{AK}+4 C_F \sqrt{L_{\delta}^p}\left(4\ln \frac{SAH}{\delta'}+5\right)HN^2S^{3/2}A\\
    &+\frac{28 C_F L_{\delta}^p\left(4\ln(4+K)+4\ln\frac{SAH}{\delta'}+5\right)}{3}H^2NS^2A\\
    =&C_F \cdot \left(\bigOTilde(H^2N S \sqrt{AK})+\bigOTilde(HN^2S^{3/2}A)+\bigOTilde(H^2NS^2A)\right).
\end{align*}
\end{lemma}

\begin{proof}
Let $\tilde{r}^k,\tilde{p}^k$, and $\pi^k$  denote the optimal $\tilde{r},\tilde{p}$, and policy in \eqref{prob.interval} in the $k$-th iteration of \cref{alg.main}, respectively.
Since outside $G^r$, we have $\tilde{r}^k\geq r$. Thus, by \cref{as.reward_bound} and \cref{remark.V_bound}, we have
\begin{align}\label{eq.temp_052101}
    V_{1,(i)}^{\pi_k}(s_1;\tilde{r}^k,\tilde{p}^k) \geq \epsilon,\quad \text{and }\ V_{1,(i)}^{\pi_k}(s_1;r,p) \geq \epsilon.
\end{align}
We have
\begin{align*}
    &\regret(K)\\
    =& \sum_{k=1}^K \left(V_{1}^*(s_1) - V_{1}^{\pi_k}(s_1)\right)\\
    =&\sum_{k=1}^K \left(V_{1}^*(s_1;r,p) - V_{1}^{\pi_k}(s_1;r,p)\right)\\
    \leq & \sum_{k=1}^K \left(V_1^{\pi_k}(s_1;\tilde{r}^k,\tilde{p}^k)-V_1^{\pi_k}(s_1;r,p)\right) \text{ (by optimism, i.e., $\tilde{r}^k$ and $\tilde{q}^k$ are optimal)}\\
    \leq & \sum_{k=1}^K C_F N \max_{i\in [N]} \abs{V_{1,(i)}^{\pi_k}(s_1;\tilde{r}^k,\tilde{p}^k)-V_{1,(i)}^{\pi_k}(s_1;r,p)}\ \text{ (by \cref{le.F_C_F} and \cref{eq.temp_052101})}\\
    =& C_F N\sum_{k=1}^K \max_{i\in[N]} \left|\E \sum_{h=1}^H\left[\left(\rtilde-r_{h,(i)}(s_h,a_h)\right)\right.\right.\\
    &+\left.\left.\sum_{s'\in \mathcal{S}}\left(\ptilde-p_h(s'|s,a)\right)\cdot V_{h+1,(i)}^{\pi_k}(s';\tilde{r}^{k-1}_{(i)},\tilde{p})\ \bigg|\ s_1,p,\pi_k\right]\right|\text{ (by Lemma~\ref{le.value_difference})}\\
    \leq & \underbrace{C_F N \sum_{k=1}^K \max_{i\in [N]}\E\left[\sum_{h=1}^H\abs{\rtilde-r_{h,(i)}(s_h,a_h)}\ \bigg|\ s_1,p,\pi_k\right]}_{\text{Term A}}\\
    &+\underbrace{C_F N \sum_{k=1}^K\max_{i\in[N]}\E\left[\sum_{h=1}^H\sum_{s'}\abs{\ptilde-p_h(s'|s,a)}\cdot \abs{V_{h+1,(i)}^{\pi_k}(s';\tilde{r}^{k-1}_{(i)},\tilde{p})}\ \bigg| s_1,p,\pi_k\right]}_{\text{Term B}}.\numberthis \label{eq.temp_110102}
\end{align*}
Since outside $G^r$ and $G^p$, we have
\begin{align}
    &\abs{\rtilde - r_{h,(i)}(s_h,a_h)} \leq 2 \rinterval,\label{eq.temp_052102}\\
    &\abs{\ptilde - p_h(s'|s,a)} \leq 2 \pinterval.\label{eq.temp_052103}
\end{align}
We have
\begin{align*}
    &\text{Term A of Eq.~\eqref{eq.temp_110102}}\\
    \leq & 2 C_F N \sum_{k=1}^K \sum_{h=1}^H \sum_{(s,a)}q_h^{\pi_k}(s,a)\rinterval \ \text{ (by \cref{eq.temp_052102})}\\
    =& 2 C_F N \sum_{k=1}^K\sum_{h=1}^H \sum_{(s,a)}q_h^{\pi_k}(s,a)\sqrt{\frac{L_{\delta}^r}{\max\{\nbar,1\}}}\ \text{ (by Eq.~\eqref{eq.def_rinterval})}\\
    \leq & 2 C_F \sqrt{L_{\delta}^r} N \left( 2H\sqrt{SAK\ln (4+K)} + SAH\left(4\ln \frac{SAH}{\delta'}+5\right)\right)\ \text{ (by Lemma~\ref{le.q_over_n})}\\
    =& 4C_F \sqrt{L_{\delta}^r \ln (4+K)}HN\sqrt{SAK} + 2 C_F \sqrt{L_{\delta}^r}\left(4\ln \frac{SAH}{\delta'}+5\right)HNSA.
\end{align*}

Since $0\leq \rtilde\leq 1$ for all $h,i,s,a$, we have
\begin{align}
    \abs{V_{h+1,(i)}^{\pi_k}(s';\tilde{r}^{k-1}_{(i)},\tilde{p})}\leq H.\label{eq.temp_052104}
\end{align}
Thus, we have
\begin{align*}
    &\text{Term B of Eq.~\eqref{eq.temp_110102}}\\
    \leq & 2 C_F N H\sum_{k=1}^K\sum_{h=1}^H \sum_{(s,a)}q_h^{\pi_k}(s,a)\sum_{s'}\pinterval\ \text{ (by \cref{eq.temp_052103,eq.temp_052104})}\\
    \leq & 4 C_F HN\sqrt{L_{\delta}^p}\sum_{k=1}^K\sum_{h=1}^H \sum_{(s,a)}q_h^{\pi_k}(s,a)\frac{1}{\sqrt{\max\{\nbar,1\}}}\sum_{s'}\sqrt{\pbar}\\
    &+\frac{28 C_F HNS L_{\delta}^p}{3}\sum_{k=1}^K\sum_{h=1}^H \sum_{(s,a)}\frac{q_h^{\pi_k}(s,a)}{\max\{\nbar,1\}}\ \text{ (by Eq.~\eqref{eq.def_pinterval} and $1-\pbar\leq 1$)}\\
    \leq &4 C_F HN\sqrt{L_{\delta}^p}\sum_{k=1}^K\sum_{h=1}^H \sum_{(s,a)}q_h^{\pi_k}(s,a)\frac{\sqrt{S}}{\sqrt{\max\{\nbar,1\}}}\sqrt{\sum_{s'}\pbar}\\
    &+\frac{28 C_F HNS L_{\delta}^p}{3}\sum_{k=1}^K\sum_{h=1}^H \sum_{(s,a)}\frac{q_h^{\pi_k}(s,a)}{\max\{\nbar,1\}}\ \text{ (by Cauchy–Schwarz inequality)}\\
    =& 4 C_F HN\sqrt{S L_{\delta}^p}\sum_{k=1}^K\sum_{h=1}^H \sum_{(s,a)}q_h^{\pi_k}(s,a)\frac{1}{\sqrt{\max\{\nbar,1\}}}\\
    &+\frac{28 C_F HNS L_{\delta}^p}{3}\sum_{k=1}^K\sum_{h=1}^H \sum_{(s,a)}\frac{q_h^{\pi_k}(s,a)}{\max\{\nbar,1\}}\ \text{ (since $\sum_{s'}\pbar = 1$)}\\
    \leq & 8 C_F \sqrt{L_{\delta}^p\ln (4+K)} H^2 N S \sqrt{AK}+4 C_F \sqrt{L_{\delta}^p}\left(4\ln \frac{SAH}{\delta'}+5\right)HN^2S^{3/2}A\\
    &+\frac{28 C_F L_{\delta}^p\left(4\ln(4+K)+4\ln\frac{SAH}{\delta'}+5\right)}{3}H^2NS^2A\ \text{ (by Lemma~\ref{le.q_over_n})}.
\end{align*}
\end{proof}

\noindent\textbf{Some auxiliary lemmas}

Define $\delta'\defeq \frac{\delta}{3}$ and 
\begin{align}\label{eq.def_Lkh}
    L_{k,h}\defeq \left\{(s,a)\ \bigg|\ \frac{1}{4}\sum_{j<k}q_h^{\pi_j}(s,a)\geq \ln \frac{SAH}{\delta'}+1\right\}.
\end{align}

The following lemmas and proofs are similar to those in \cite{efroni2019tight,zanette2019tighter} with different notations. For ease of reading, we provide the full proof using the notation of this paper.
\begin{lemma}\label{le.q_over_n}
If  outside the failure event $G^n$, then
\begin{align}
    &\sum_{k=1}^K\sum_{h=1}^H\sum_{s,a}q_h^{\pi_k}(s,a)\sqrt{\frac{1}{\max\{\nbar,1\}}}\leq 2H\sqrt{SAK\ln (4+K)} + SAH\left(4\ln \frac{SAH}{\delta'}+5\right),\label{eq.temp_110105}\\
    &\sum_{k=1}^K\sum_{h=1}^H\sum_{s,a}\frac{q_h^{\pi_k}(s,a)}{\max\{\nbar,1\}}\leq SAH\left(4\ln(4+K)+4\ln\frac{SAH}{\delta'}+5\right).\label{eq.temp_110106}
\end{align}
\end{lemma}
\begin{proof}
We have
\begin{align*}
    &\sum_{k=1}^K\sum_{h=1}^H\sum_{s,a}q_h^{\pi_k}(s,a)\sqrt{\frac{1}{\max\{\nbar,1\}}}\\
    \leq & \sum_{k=1}^K\sum_{h=1}^H\sum_{s,a}q_h^{\pi_k}(s,a)\sqrt{\frac{1}{\nbar}}\ \text{ (since $\nbar\geq 1$ by Lemma~\ref{le.n_4_q})}\\
    \leq & \sum_{k=1}^K\sum_{h=1}^H\sum_{(s,a)\in L_{k,h}}q_h^{\pi_k}(s,a)\sqrt{\frac{1}{\nbar}} + \sum_{k=1}^K\sum_{h=1}^H\sum_{(s,a)\notin L_{k,h}}q_h^{\pi_k}(s,a).\numberthis \label{eq.temp_103005}
\end{align*}
By Eq.~\eqref{eq.def_q}, we have
\begin{align}
    \sum_{k=1}^K\sum_{h=1}^H\sum_{(s,a)\in L_{k,h}}q_h^{\pi_k}(s,a)\leq \sum_{k=1}^K\sum_{h=1}^H\sum_{(s,a)}q_h^{\pi_k}(s,a)=KH.\label{eq.temp_110101}
\end{align}
For the first term of Eq.~\eqref{eq.temp_103005}, we have
\begin{align*}
    &\sum_{k=1}^K\sum_{h=1}^H\sum_{(s,a)\in L_{k,h}}q_h^{\pi_k}(s,a)\sqrt{\frac{1}{\nbar}}\\
    =&\sum_{k=1}^K\sum_{h=1}^H\sum_{(s,a)\in L_{k,h}}\sqrt{q_h^{\pi_k}(s,a)}\sqrt{\frac{q_h^{\pi_k}(s,a)}{\nbar}}\\
    \leq & \sqrt{\sum_{k=1}^K\sum_{h=1}^H\sum_{(s,a)\in L_{k,h}}q_h^{\pi_k}(s,a)}\cdot \sqrt{\sum_{k=1}^K\sum_{h=1}^H\sum_{(s,a)\in L_{k,h}}\frac{q_h^{\pi_k}(s,a)}{\nbar}} \ \text{ (by Cauchy-Schwarz inequality)}\\
    \leq & \sqrt{KH}\sqrt{4SAH\ln (4+K)}\ \text{ (by Eq.~\eqref{eq.temp_110101} and Lemma~\ref{le.shlnk})}\\
    =& 2H\sqrt{SAK\ln (4+K)}.\numberthis \label{eq.temp_103007}
\end{align*}
By Eq.~\eqref{eq.temp_103007}, Eq.~\eqref{eq.temp_103005}, and Lemma~\ref{le.sum_not_in_L}, we can get Eq.~\eqref{eq.temp_110105}.
It remains to prove Eq.~\eqref{eq.temp_110106}. To that end, we have
\begin{align*}
    &\sum_{k=1}^K\sum_{h=1}^H\sum_{s,a}\frac{q_h^{\pi_k}(s,a)}{\max\{\nbar,1\}}\\
    \leq & \sum_{k=1}^K\sum_{h=1}^H\sum_{s,a}\frac{q_h^{\pi_k}(s,a)}{\nbar}\ \text{ (since $\nbar\geq 1$ by Lemma~\ref{le.n_4_q})}\\
    \leq & \sum_{k=1}^K\sum_{h=1}^H\sum_{(s,a)\in L_{k,h}}\frac{q_h^{\pi_k}(s,a)}{\nbar}+\sum_{k=1}^K\sum_{h=1}^H\sum_{(s,a)\notin L_{k,h}}q_h^{\pi_k}(s,a)\\
    \leq & 4SAH\ln(4+K)+SAH\left(4\ln \frac{SAH}{\delta'}+5\right)\ \text{ (by Lemma~\ref{le.sum_not_in_L} and Lemma~\ref{le.shlnk})}\\
    = & SAH\left(4\ln(4+K)+4\ln\frac{SAH}{\delta'}+5\right).
\end{align*}
Thus, we have proven Eq.~\eqref{eq.temp_110106}. The result of this lemma thus follows.
\end{proof}

\begin{lemma}\label{le.n_4_q}
If outside the failure event $G^n$, for any $(s,a)\in L_{k,h}$, we must have
\begin{align*}
    \nbar\geq \max\left\{\frac{1}{4}\sum_{j\leq k}q_h^{\pi_j}(s,a),\ 1\right\}.
\end{align*}
\end{lemma}
\begin{proof}
We have
\begin{align*}
    \nbar>&\frac{1}{4}\sum_{j<k}q_h^{\pi_j}(s,a)+\frac{1}{4}\sum_{j<k}q_h^{\pi_j}(s,a)-\ln \frac{SAH}{\delta'}\text{ (recall the definition of $G^n$ in Eq.~\eqref{eq.def_Gn})}\\
    \geq &\frac{1}{4}\sum_{j<k}q_h^{\pi_j}(s,a)+1\quad \text{ (since $(s,a)\in L_{k,h}$)}\\
    \geq & \max\left\{\frac{1}{4}\sum_{j\leq k}q_h^{\pi_j}(s,a),\ 1\right\}\quad \text{ (since $q_h^{\pi_k}\leq 1$)}.
\end{align*}
\end{proof}

\begin{lemma}\label{le.not_in_L}
For any $(s,a)\not \in L_{k,h}$, we must have
\begin{align*}
    \sum_{j\leq k}q_h^{\pi_j}(s,a)\leq 4\ln \frac{SAH}{\delta'}+5.
\end{align*}
\end{lemma}
\begin{proof}
Since $(s,a)\notin L_{k,h}$, we have
\begin{align*}
    \frac{1}{4}\sum_{j<k}q_h^{\pi_j}(s,a)<\ln \frac{SAH}{\delta'}+1.
\end{align*}
Thus, we have
\begin{align*}
    \sum_{j<k}q_h^{\pi_j}(s,a)<4\ln \frac{SAH}{\delta'}+4.
\end{align*}
Because $q_h^{\pi_k}\leq 1$, the result of this lemma thus follows.
\end{proof}

\begin{lemma}\label{le.sum_not_in_L}
We have
\begin{align*}
    \sum_{k=1}^K\sum_{h=1}^H\sum_{(s,a)\notin L_{k,h}}q_h^{\pi_k}(s,a)\leq SAH\left(4\ln \frac{SAH}{\delta'}+5\right).
\end{align*}
\end{lemma}
\begin{proof}
Define
\begin{align}
    k_{s,a,h}\defeq \begin{cases}
        0,&\text{ if }\left\{k^*\in [K]\ |\ (s,a)\notin L_{k^*,h}\right\}=\varnothing,\\
        \max\left\{k^*\in [K]\ |\ (s,a)\notin L_{k^*,h}\right\}&\text{ otherwise}.
    \end{cases}\label{eq.temp_102101}
\end{align}
By the definition of $L_{k,h}$ in Eq.~\eqref{eq.def_Lkh}, we know that
\begin{align}
    (s,a)\notin L_{k,h} \text{ for all }k\leq k_{s,a,h}.\label{eq.temp_102102}
\end{align}
Therefore, we have
\begin{align*}
    \sum_{k=1}^K\sum_{h=1}^H\sum_{(s,a)\notin L_{k,h}}q_h^{\pi_k}(s,a)=&\sum_{(s,a)}\sum_{k=1}^K\sum_{h=1}^H q_h^{\pi_k}(s,a)\mathbbm{1}\left((s,a)\notin L_{k,h}\right)\\
    = &\sum_{(s,a)}\sum_{h=1}^H\sum_{k=1}^{k_{s,a,h}}q_h^{\pi_k}(s,a)\text{ (by Eq.~\eqref{eq.temp_102101})}\\
    \leq & SAH\left(4\ln \frac{SAH}{\delta'}+5\right)\text{ (by Eq.~\eqref{eq.temp_102102} and Lemma~\ref{le.not_in_L})}.
\end{align*}
\end{proof}

\begin{lemma}\label{le.shlnk}
If outside  the failure event $G^n$, we must have
\begin{align*}
    \sum_{k=1}^K\sum_{h=1}^H\sum_{(s,a)\in L_{k,h}}\frac{q_h^{\pi_k}(s,a)}{\nbar}\leq 4SAH\ln (4+K).
\end{align*}
\end{lemma}
\begin{proof}
We have
\begin{align*}
    \sum_{k=1}^K\sum_{h=1}^H\sum_{(s,a)\in L_{k,h}}\frac{q_h^{\pi_k}(s,a)}{\nbar}\leq &4\sum_{k=1}^K\sum_{h=1}^H\sum_{(s,a)\in L_{k,h}}\frac{q_h^{\pi_k}(s,a)}{\sum_{j\leq k}q_h^{\pi_j}(s,a)}\text{ (by Lemma~\ref{le.n_4_q})}\\
    =&4\sum_{(s,a)}\sum_{h=1}^H\sum_{k=1}^K \frac{q_h^{\pi_k}(s,a)}{\sum_{j\leq k}q_h^{\pi_j}(s,a)}\mathbbm{1}\left((s,a)\in L_{k,h}\right).\numberthis \label{eq.temp_103002}
\end{align*}
For fixed $s,a,h$, if
\begin{align*}
    \left\{k=1,2,\cdots, \ |\ (s,a)\in L_{k,h}\right\}\neq \varnothing,
\end{align*}
then by the monotonicity of the size of $L_{k,h}$ with respect to $k$, we can define
\begin{align*}
    \kthres\defeq \min\left\{k=1,2,\cdots, \ |\ (s,a)\in L_{k,h}\right\}.
\end{align*}
Thus, we have
\begin{align*}
    \sum_{k=1}^K\frac{q_h^{\pi_k}(s,a)}{\sum_{j\leq k}q_h^{\pi_j}(s,a)}\mathbbm{1}\left((s,a)\in L_{k,h}\right)=&\sum_{k=\kthres}^K \frac{q_h^{\pi_k}(s,a)}{\sum_{j\leq k}q_h^{\pi_j}(s,a)}\\
    \leq & \sum_{k=\kthres}^K \frac{q_h^{\pi_k}(s,a)}{4 + \sum_{\kthres\leq j\leq k}q_h^{\pi_j}(s,a)}.\numberthis \label{eq.temp_103001}
\end{align*}
The last inequality is because $(s,a)\in L_{\kthres,h}$, by the definition of $L_{k,h}$ in Eq.~\eqref{eq.def_Lkh}, we have
\begin{align*}
    \frac{1}{4}\sum_{j<\kthres}q_h^{\pi_j}(s,a)\geq \ln \frac{SAH}{\delta'}+1\geq 1.
\end{align*}
Define functions
\begin{align*}
    &G_{s,a,h}(x)\defeq \left(x-\floor{x}\right)\cdot q_h^{\pi_{\ceil{x}}}(s,a)+\sum_{\kthres\leq j\leq \floor{x}}q_h^{\pi_j}(s,a),\numberthis \label{eq.def_G}\\
    &g_{s,a,h}(x)\defeq q_h^{\pi_{\ceil{x}}}(s,a).\numberthis \label{eq.def_g}
\end{align*}
Roughly speaking, $G_{s,a,h}(x)$ is the linear interpolation of the sum of $q_h^{\pi_j}(s,a)$, and $g_{s,a,h}(x)$ is the step function whose steps are $q_h^{\pi_j}(s,a)$. We can easily check that when $x$ is not an integer,
\begin{align}
    \frac{\partial G_{s,a,h}(x)}{\partial x}=g_{s,a,h}(x).\label{eq.G_partial_g}
\end{align}
Notice that the not-differentiable points (i.e., when $x$ is an integer) of $G_{s,a,h}(s,a)(x)$ are countable and will not affect the following calculation.

\begin{align*}
    &\sum_{k=\kthres}^K\frac{q_h^{\pi_k}(s,a)}{4+\sum_{\kthres\leq j\leq k}q_h^{\pi_j}(s,a)}\\
    =&\int_{\kthres-1}^K \frac{g_{s,a,h}(x)}{4+G_{s,a,h}(\ceil{x})}dx\ \text{ (by Eq.~\eqref{eq.def_G} and Eq.~\eqref{eq.def_g})}\\
    \leq & \int_{\kthres-1}^K \frac{g_{s,a,h}(x)}{4+G_{s,a,h}(x)}dx\ \text{ (since $G_{s,a,h}(\cdot)$ is monotone increasing)}\\
    = & \ln(4+G_{s,a,h}(K))-\ln (4+G_{s,a,h}(\kthres-1)) \ \text{ (by Eq.~\eqref{eq.G_partial_g})}\\
    \leq & \ln \left(4 + \sum_{\kthres\leq j\leq K}q_h^{\pi_j}(s,a)\right)\\
    \leq & \ln (4+K).\numberthis \label{eq.temp_103003}
\end{align*}
By Eq.~\eqref{eq.temp_103003}, Eq.~\eqref{eq.temp_103002}, and Eq.~\eqref{eq.temp_103001}, the result of this lemma thus follows.
\end{proof}

\section{Proof of Theorem~\ref{th.PAC}}\label{proof.PAC}

\begin{proof}
For the proof of PAC guarantee, we have
\begin{align*}
    V_1^{*,F}(s_1)-V_1^{\piMix}(s_1)=V_1^{*,F}(s_1)-F\left(\left[\frac{1}{K}\sum_{k=1}^K V_{1,(i)}^{\pi_k}(s_1)\right]_{i=1,\cdots,N}\right).
\end{align*}
Because $F$ is a concave function, by Jensen's inequality, we have 
\begin{align*}
    F\left(\left[\frac{1}{K}\sum_{k=1}^K V_{1,(i)}^{\pi_k}(s_1)\right]_{i=1,\cdots,N}\right) \geq \frac{1}{K}\sum_{i=1}^K F\left(\left[V_{1,(i)}^{\pi_k}(s_1)\right]_{i=1,\cdots,N}\right)=\frac{1}{K}\sum_{i=1}^K V_1^{\pi_k,F}(s_1).
\end{align*}
Thus, we have $V_1^{*,F}(s_1)-V_1^{\piMix}(s_1)\leq \frac{1}{K}\sum_{k=1}^K \left(V_1^{*,F}(s_1) - V_1^{\pi_k,F}(s_1)\right) = \frac{\regret(K)}{K}$. By \cref{th.main}, we know that  with high probability $\regret(K)=C_{F}\cdot\left(\bigOTilde(H^2N S \sqrt{AK})+\bigOTilde(HN^2S^{3/2}A)+\bigOTilde(H^2NS^2A)\right)$. Thus, by letting 
\begin{align*}
    \varepsilon = \frac{C_{F}\cdot\left(\bigOTilde(H^2N S \sqrt{AK})+\bigOTilde(HN^2S^{3/2}A)+\bigOTilde(H^2NS^2A)\right)}{K},
\end{align*}
we can get
\begin{align*}
    K = C_F \max\left\{\bigOTilde(H^4 N^2 S^2 A/\varepsilon^2),\ \bigOTilde(H^2 N^4 S^3 A^2 /\varepsilon^2),\ \bigOTilde(H^2 N^2 S^4 A^2 / \varepsilon^2)\right\}.
\end{align*}
Notice that for each episode we have $H$ samples. Thus, the total number of samples is
\begin{align*}
    C = C_F \max\left\{\bigOTilde(H^5 N^2 S^2 A/\varepsilon^2),\ \bigOTilde(H^3 N^4 S^3 A^2 /\varepsilon^2),\ \bigOTilde(H^3 N^2 S^4 A^2 / \varepsilon^2)\right\}.
\end{align*}
The result thus follows.
\end{proof}

\section{Proof of Theorem~\ref{th.offline}}\label{proof:offline}

Recalling the definition of suboptimality, we get \cref{eq.offline_decomp}.
Term~2 of \cref{eq.offline_decomp} is non-positive because $\hat{\pi}$ is the solution of \cref{op.offline}.

We now bound the Term 3. First, we specify the value of $b_h^r(s,a,\delta)$ and $b_h^p(s,a,s^{\prime},\delta)$. For the offline setup, we denote $n_h(s,a,s^{\prime})$ as the empirical value within the dataset. Now, set $b_h^r(s,a,\delta)$ as the value in (\ref{eq.r_interval}) and $b_h^p(s,a,s^{\prime},\delta)$ as the value in (\ref{eq.def_pinterval}) respectively. From the Value-difference Lemma, for any $i$, we have
\begin{align}
V_{1,(i)}^{\hat{\pi}}(s,\underline{r},\overline{p})-V_{1,(i)}^{\hat{\pi}}(s,r,p)&=\mathbbm{E}_{p,\hat{\pi}}\left[\sum_{h=1}^{H}(\underline{r}_{(i),h}(s_h,a_h)-r_{(i),h}(s_h,a_h)|s_1=s\right]\nonumber\\
& +\mathbbm{E}_{p,\hat{\pi}}[\sum_{h=1}^H\sum_{s_{h+1}^{\prime}}(\overline{p}_h(s_h,a_h,s_{h+1}^{\prime})-p(s_h,a_h,a_{h+1}^{\prime}))V_{h+1,(i)}]
\end{align}
From Lemma~\ref{lem:gp}, \ref{lem:gn}, and \ref{lem:gr}, we have $|\overline{r}_{(i),h}(s,a)-r_{(i),h}(s,a)|\leq b_h^r(s,a,\delta)$, and $|\overline{p}(s,a,s^{\prime})-p(s,a,s^{\prime})|\leq b_h^p(s,a,s^{\prime})$ with probability $1-\delta$. Since, $V_{h+1,(i)}\leq H$. Thus,
\begin{align}
&V_{1,(i)}^{\hat{\pi}}(s,\underline{r},\overline{p})-V_{1,(i)}^{\hat{\pi}}(s,r,p)\leq\nonumber\\& \mathbbm{E}_{p,\hat{\pi}}\left[\sum_{h=1}^H(\underline{r}_{(i),h}(s_h,a_h)-\overline{r}_{(i),h}(s_h,a_h)+b_h^r(s_h,a_h,\delta)+H\sum_{s_{h+1}^{\prime}}b_h^p(s_h,a_h,s_{h+1}^{\prime},\delta))\right]
\end{align}
Now, by the definition of $\underline{r}$, we can bound the above by $0$. Finally, using the fact that $F(\cdot)$ is monotone increasing, we can conclude that Term~3 is bounded by $0$.

It remains to estimate Term~1. To that end, when the event $\mathcal{E}$ (defined in \cref{def.offline_event}) happens, we have
\begin{align*}
    &\abs{\underline{r}_{h,(i)}(s,a)-r_{h,(i)}(s,a)} \\
    =& \abs{\overline{r}_{h,(i)}(s,a) - r_{h,(i)}(s,a)-b_h^r(s,a,\delta) -H \sum_{s'\in \mathcal{S}}b_h^p(s,a,s',\delta)}\ \text{ (by the definition of $\underline{r}$)}\\
    \leq & 2 b_h^r(s,a,\delta) + H \sum_{s'\in \mathcal{S}} b_h^p(s,a,s',\delta) \ \text{ (by \cref{def.offline_event} and the triangle inequality)}.\numberthis \label{eq.temp_052301}
\end{align*}

By \cref{as.offline_reward}, we have 
\begin{align}
    V_{h,(i)}^{\pi^*,F}(s,\underline{r},\overline{p}) \in [\epsilon, H].\label{eq.temp_052302}
\end{align}
Thus, we can apply \cref{le.F_C_F}. Specifically, under \cref{as.offline_reward} and when the event $\mathcal{E}$ happens, we have
\begin{align*}
    &\text{Term 1 of \cref{eq.offline_decomp}}\\
    =& V_1^{\pi^*,F}(s_1;r,p) - V_1^{\pi^*,F}(s_1;\underline{r},\overline{p})\\
    \leq & C_F N \max_{i\in [N]} \abs{V_{i,(i)}^{\pi^*}(s_1;r,p)-V_{1,(i)}^{\pi^*}(s_1;\underline{r},\overline{p})}\ \text{ (by \cref{le.F_C_F})}\\
    =& C_F N \max_{i\in [N]} \left| \E\left[\underline{r}_{h,(i)}(s_h,a_h)-r_{h,(i)}(s_h,a_h)\right]   \right.\\
    & + \left. \E\left[\sum_{h=1}^H \sum_{s'\in \mathcal{S}} \left(\overline{p}_h(s'|s_h,a_h)-p_h(s'|s_h,a_h)\right) V_{h+1,(i)}^{\pi^*}(s';\underline{r},\overline{p})\right] \right| \ \text{ (by \cref{le.value_difference})}\\
    \leq & C_F N \max_{i\in [N]} \E \sum_{h=1}^H \abs{\underline{r}_{h,(i)}(s_h,a_h)-r_{h,(i)}(s_h,a_h)} \\
    &+ C_F N \max_{i\in [N]} \sum_{h=1}^H \sum_{s'\in \mathcal{S}}\abs{\left(\overline{p}_h(s'|s_h,a_h)-p_h(s'|s_h,a_h)\right) V_{h+1,(i)}^{\pi^*}(s';\underline{r},\overline{p})} \\ 
     & \qquad \text{ (by the triangle inequality)}\\
    \leq & C_F N \E \left[\sum_{h=1}^H \left(2 b_h^r(s_h,a_h,\delta) + H \sum_{s'\in \mathcal{S}} b_h^p(s_h,a_h,s',\delta)\right)\right]\\
    & + C_F N \E \left[ \sum_{h=1}^H \sum_{s'\in \mathcal{S}} b_h^p(s_h,a_h,s',\delta) H \right]\ \text{ (by \cref{eq.temp_052301,eq.temp_052302} and \cref{def.offline_event})}\\
    =& 2 C_F N \E \left[\sum_{h=1}^H \left( b_h^r(s_h,a_h,\delta) + H \sum_{s'\in \mathcal{S}} b_h^p(s_h,a_h,s',\delta)\right)\right].
\end{align*}
(The expectation $\E$ in the above equation is on the trajectories with optimal policy $\pi^*$ on the true MDP with $r$ and $p$.) The result of \cref{th.offline} thus follows.
\section{Details of Section~\ref{sec:policy_gradient}}\label{app.gradient}

The following proposition gives an estimation of the gradient based on the samples.

\begin{proposition}\label[proposition]{prop.gradient}
After collecting a set $\mathcal{D}$ of trajectories (with the policy $\pi_\theta$) where each trajectory $\tau\in \mathcal{D}$ contains the information $(s_h^{\tau},a_h^{\tau},\bm{r}_h^{\tau})_{h=1,2,\cdots,H}$, then $\bm{g}\in \mathds{R}^d$ is an unbiased\footnote{An unbiased estimation means that when $\abs{\mathcal{D}}\to \infty$, the estimated value approaches the true value.} estimation of the gradient $\nabla_{\theta} V_1^{\pi_{\theta},F}(s_1)$.

\begin{align}
&
    \bm{g}_{\text{max-min}} = \frac{1}{\abs{\mathcal{D}}}\sum_{\tau\in \mathcal{D}}  \sum_{h=1}^H R_{(\hat{i}_{\bm{\theta}}^*)}(\tau)\nabla_{\bm{\theta}}\log \pi_{\bm{\theta}}(a_h^{\tau} | s_h^{\tau}),\text{ where } \hat{i}^*_{\bm{\theta}_l}\defeq  \argmin_{i\in [N]} \sum_{\tau\in \mathcal{D}} R_{(i)}(\tau).\nonumber\\
&
    \bm{g}_{\text{proportional}} = \sum_{i=1}^N \frac{\sum_{\tau\in \mathcal{D}}\sum_{h=1}^H R_{(i)}(\tau)\nabla_{\bm{\theta}}\log \pi_{\bm{\theta}}(a_h^{\tau} | s_h^{\tau})}{\sum_{\tau\in \mathcal{D}}R_{(i)}(\tau)}.\label{eq.g_proportional}\\
     &\bm{g}_{\alpha} = \abs{\mathcal{D}}^{\alpha-1}\sum_{i=1}^N \frac{\sum_{\tau\in \mathcal{D}}\sum_{h=1}^H R_{(i)}(\tau)\nabla_{\bm{\theta}}\log \pi_{\bm{\theta}}(a_h^{\tau} | s_h^{\tau})}{\left(\sum_{\tau\in \mathcal{D}}R_{(i)}(\tau)\right)^{\alpha}}.\nonumber
\end{align}
\end{proposition}

\begin{proof}

Based on the chain rule, we have the following results.

1. When $F=F_{\text{max-min}}$, let $i_{\bm{\theta}}^*\defeq \argmin_{i\in [N]} V_{1,(i)}^{\pi_{\bm{\theta}},F}(s_1)$:
\begin{align}
    \nabla_{\bm{\theta}}V_{1}^{\pi_{\bm{\theta}},F}(s_1)=\nabla_{\bm{\theta}}V_{1, (i_{\bm{\theta}}^*)}^{\pi_{\bm{\theta}}}(s_1).\label{eq.temp_052316}
\end{align}
Notice that $\nabla_{\bm{\theta}}\ i_{\bm{\theta}}^* = \bm{0}$ almost everywhere if $V_{1,(i)}^{\pi_{\bm{\theta}},F}(s_1)$ is continuous w.r.t. $\bm{\theta}$.

2. When $F=F_{\text{proportional}}$:
\begin{align}
    \nabla_{\bm{\theta}}V_{1}^{\pi_{\bm{\theta}},F}(s_1)=\nabla_{\bm{\theta}}\sum_{i=1}^N \log V_{1, (i)}^{\pi_{\bm{\theta}}}(s_1)=\sum_{i=1}^N \frac{\nabla_{\bm{\theta}}V_{1, (i)}^{\pi_{\bm{\theta}}}(s_1)}{V_{1, (i)}^{\pi_{\bm{\theta}}}(s_1)}.\label{eq.temp_052317}
\end{align}

3. When $F=F_{\alpha}$:
\begin{align}
    \nabla_{\bm{\theta}}V_{1}^{\pi_{\bm{\theta}},F}(s_1)=\nabla_{\bm{\theta}}\sum_{i=1}^N \frac{1}{1-\alpha } \left(V_{1, (i)}^{\pi_{\bm{\theta}}}(s_1)\right)^{1-\alpha} = \sum_{i=1}^N \left(V_{1, (i)}^{\pi_{\bm{\theta}}}(s_1)\right)^{-\alpha}\nabla_{\bm{\theta}}V_{1, (i)}^{\pi_{\bm{\theta}}}(s_1).\label{eq.temp_052318}
\end{align}

It remains to approximate $\nabla_{\bm{\theta}}V_{1, (i)}^{\pi_{\bm{\theta}}}(s_1)$ and $V_{1, (i)}^{\pi_{\bm{\theta}}}(s_1)$ in the above equations. To that end, noticing that $V_{1, (i)}^{\pi_{\bm{\theta}}}(s_1) = \E_{\tau} R_{(i)} (\tau)$, we can approximate $V_{1, (i)}^{\pi_{\bm{\theta}}}(s_1)$ by the empirical average of $R_{(i)}$, i.e.,
\begin{align}
    \frac{1}{\abs{\mathcal{D}}}\sum_{\tau\in \mathcal{D}} R_{(i)}(\tau). \label{eq.temp_052314}
\end{align}

Before calculating $\nabla_{\bm{\theta}}V_{1, (i)}^{\pi_{\bm{\theta}}}(s_1)$, we first list some equations that will be used later.

1. Probability of a trajectory:

\begin{align}
    \Pr(\tau|\bm{\theta})=\prod_{h=1}^H p_h(s_{h+1}|s_h,a_h) \pi_{\bm{\theta}}(a_h|s_h). \label{eq.temp_052311}
\end{align}

2. The log-derivative trick:

\begin{align}
    \nabla_{\bm{\theta}}\Pr(\tau|\bm{\theta})=\Pr(\tau|\bm{\theta}) \cdot \nabla_{\bm{\theta}} \log \Pr(\tau| \bm{\theta}). \label{eq.temp_052312}
\end{align}

3. Log-probability of a trajectory:

\begin{align*}
    \log \Pr(\tau | \bm{\theta})=& \log \prod_{h=1}^H p_h(s_{h+1} | s_h, a_h) \pi_{\bm{\theta}}(a_h | s_h)\ \text{ (by \cref{eq.temp_052311})}\\
    =& \sum_{h=1}^H \log p_h(s_{h+1} | s_h,a_h) + \log \pi_{\bm{\theta}}(a_h | s_h).
\end{align*}
Thus, we have
\begin{align}
    \nabla_{\theta} \log \Pr(\tau | \theta) = \sum_{h=1}^H \nabla_{\theta} \log \pi_{\theta} (a_h | s_h).\label{eq.temp_052313}
\end{align}
Notice that to get the above equation, we use the fact that the transition probability $p$ is irrelevant to $\bm{\theta}$.

Now we are ready to calculate $\nabla_{\bm{\theta}}V_{1, (i)}^{\pi_{\bm{\theta}}}(s_1)$. We have
\begin{align*}
    \nabla_{\bm{\theta}} V_{1,(i)}^{\pi_{\theta}}(s_1)=&\nabla_{\bm{\theta}} \E_{\tau} R_{(i)}(\tau)\\
    =& \nabla_{\bm{\theta}} \int_{\tau} \Pr(\tau | \bm{\theta}) R_{(i)}(\tau)\\
    =& \int_{\tau} \nabla_{\bm{\theta}} \Pr(\tau | \bm{\theta}) R_{(i)}(\tau)\\
    =& \int_{\tau} \Pr(\tau | \theta) \nabla_{\bm{\theta}} \log (\Pr(\tau | \bm{\theta})) R_{(i)}(\tau)\ \text{ (by \cref{eq.temp_052312})} \\
    = & \E_{\tau} \nabla_{\bm{\theta}} \log (\Pr(\tau | \bm{\theta})) R_{(i)}(\tau) \\
    =& \E_{\tau}R_{(i)}(\tau)\nabla_{\bm{\theta}} \log \Pr(\tau | \bm{\theta}) \\
    = & \E_{\tau} R_{(i)}(\tau)\sum_{h=1}^H \nabla_{\bm{\theta}} \log \pi_{\bm{\theta}}(a_h | s_h)\ \text{ (by \cref{eq.temp_052313})}.
\end{align*}
Thus, we can approximate $\nabla_{\bm{\theta}} V_{1,(i)}^{\pi_{\bm{\theta}}}(s_1)$ by the following empirical average:
\begin{align}
    \frac{1}{\abs{\mathcal{D}}}\sum_{\tau \in \mathcal{D}}R_{(i)}(\tau)\sum_{h=1}^H \nabla_{\bm{\theta}} \log \pi_{\bm{\theta}}(a_h | s_h).\label{eq.temp_052315}
\end{align}
The result of this proposition thus follows by plugging the empirical estimation \cref{eq.temp_052314,eq.temp_052315} into \cref{eq.temp_052316,eq.temp_052317,eq.temp_052318}.
\end{proof}

As an example, we show the whole algorithm for max-min fairness in \cref{alg.policy}.

\begin{algorithm}
\caption{Policy Gradient for Max-Min Fairness}\label{alg.policy}
\begin{algorithmic}[1]
\State \textbf{Initialize:} $\bm{\theta}_0\in \mathds{R}^d$, step size $\alpha'> 0$.
\For{each iteration $l=0,1,2,\cdots,$}
\State Collect a set of trajectory $\mathcal{D}$ by using $\pi_{\bm{\theta}_l}$ where each trajectory $\tau\in \mathcal{D}$ contains the information $(s_h,a_h,\bm{r}_h)_{h=1,2,\cdots,H}$.
\State For each collected trajectory, calculate its total reward for each agent
\State For each agent $i$, get an estimation of its own value function $\hat{V}_{1,(i)}^{\bm{\theta}_l}(s_1)\gets \frac{1}{\abs{\mathcal{D}}}\sum_{\tau\in \mathcal{D}} R_{(i)}(\tau)$.
\State Select the agent with the minimum estimated value $\hat{i}^*_{\bm{\theta}_l}\gets \argmin_{i\in \{1, 2, \cdots, N\}} \hat{V}_{1,(i)}^{\bm{\theta}_l}(s_1)$.
\State Calculate the estimated gradient $\bm{g}\in \mathds{R}^d$ by 
\begin{align*}
    \bm{g}\gets \frac{1}{\abs{\mathcal{D}}}\sum_{\tau\in \mathcal{D}} R_{(\hat{i}_{\bm{\theta}_l}^*)}(\tau) \sum_{h=1}^H\nabla_{\bm{\theta}_l}\log \pi_{\bm{\theta}_l}(a_h | s_h).
\end{align*}
\State Update the parameters $\bm{\theta}_{l+1}\gets \bm{\theta}_l +  \alpha' g$.
\EndFor
\end{algorithmic}

\end{algorithm}
\section{Simulation Results}\label{app:simulation}

We run some simulations on a synthetic setup where $S=A=N=2$ and $H=3$. Each term of the transition probability $p$ is \emph{i.i.d.} uniformly generated between $[0, 1]$, and then we normalize $p$ to make sure that $\sum_{s'\in \mathcal{S}}p(s,a,s')=1$. Every term of the true immediate reward $r$ is \emph{i.i.d.} uniformly generated between $[0.15, 0.95]$. Each noisy observation of an immediate reward is drawn from a uniform distribution centered at its true value within the range of $\pm 0.05$ (thus all noisy observations are in $[0.1, 1]$).

\begin{figure}[ht]
    \centering
    \includegraphics[width=0.9\textwidth]{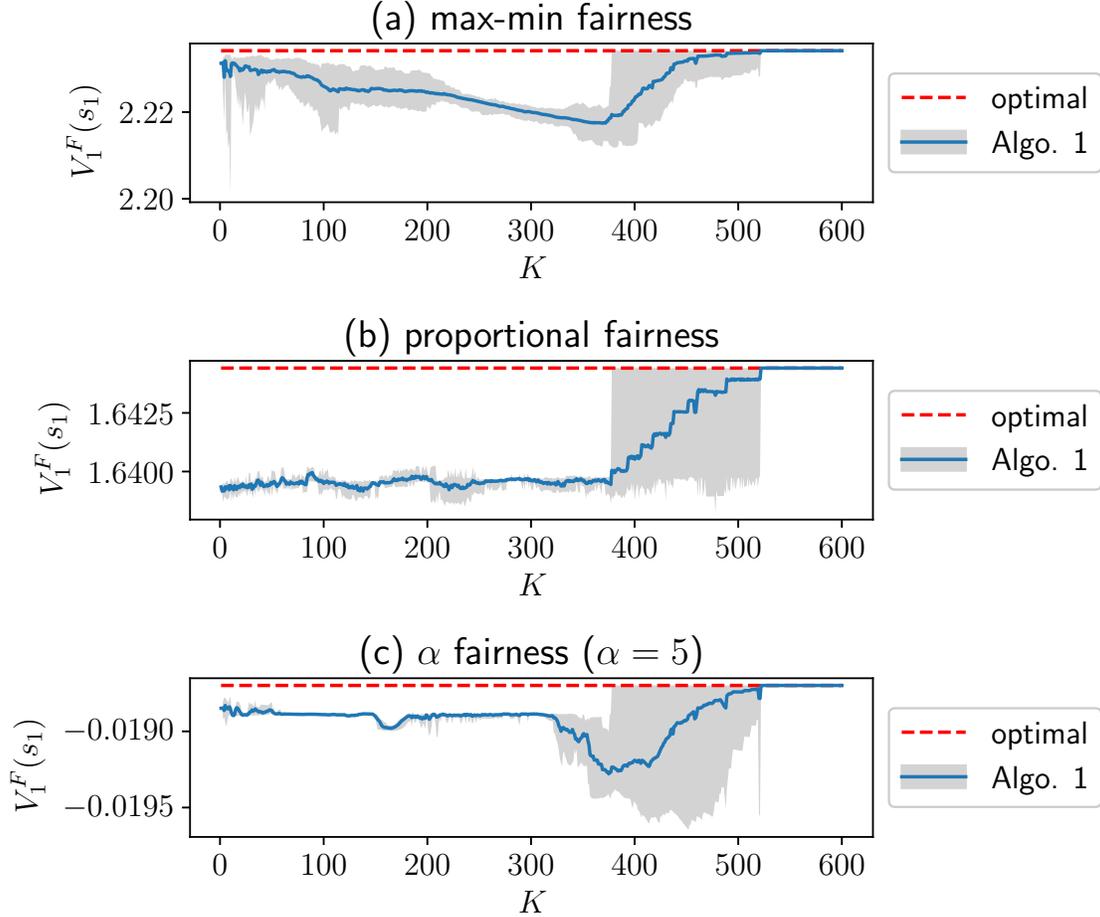}
    \caption{Curves of $V_1^F(s_1)$ of \cref{alg.main} w.r.t. $K$ for different fair objectives. The solid blue curve is the average of 10 runs with different seeds. The shaded part denotes the range of these 10 runs (i.e., the range between the max and the min value).}
    \label{fig:opti}
\end{figure}

In \cref{fig:opti}, we plot the curves of $V_1^F(s_1)$ of the optimal policy (dashed red curves) and the curves of the policy calculated by \cref{alg.main} (the blue curves) for different fair objectives. We can see that for all three different fair objectives, the solution of \cref{alg.main} becomes very close to the optimal one after $K\geq 550$.
This validates our theoretical result that the regret scales sub-linearly ($\bigOTilde(\sqrt{K})$) since the average regret $\bigOTilde(\frac{\sqrt{K}}{K})$ goes to zero when $K$ becomes larger.


\begin{figure}[ht]
    \centering
    \includegraphics[width=0.9\textwidth]{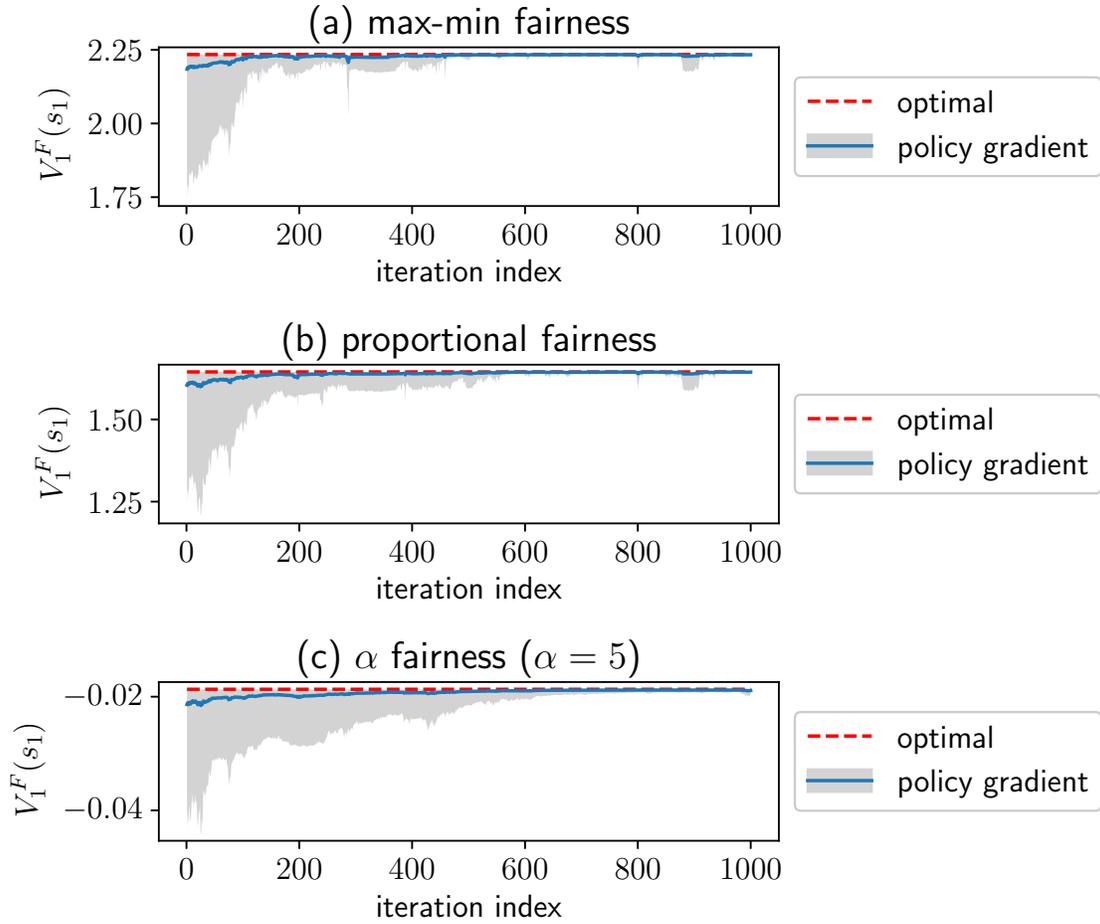}
    \caption{Curves of $V_1^F(s_1)$ of policy gradient w.r.t. the number of iterations for different fair objectives. The solid blue curve is the average of 10 runs with different seeds. The shaded part denotes the range of these 10 runs (i.e., the range between the max and the min value).}
    \label{fig:grad}
\end{figure}

In \cref{fig:grad}, we plot the curves of $V_1^F(s_1)$ of the optimal policy (dashed red curves) and the curves of the policy calculated by the policy gradient method (the blue curves). We use a two-layer fully-connected neural network with ReLU (rectified linear unit) as the policy model. During each iteration of the policy gradient algorithm, $20$ trajectories are generated and collected under the current policy. As shown by \cref{fig:grad}, such a policy gradient method can achieve the nearly optimal solution within 1000 iterations.

\end{document}